\documentclass[10pt,journal,compsoc]{IEEEtran}
\ifCLASSOPTIONcompsoc
  \usepackage[nocompress]{cite}
\else
  \usepackage{cite}
\fi
\ifCLASSINFOpdf
\else
\fi
\usepackage{times}
\usepackage{epsfig}
\usepackage{graphicx}
\usepackage{amsmath}
\usepackage{amssymb}
\usepackage{bm}
\usepackage{url}
\usepackage{makecell}
\usepackage{float}
\usepackage{amsthm}
\usepackage[hang]{footmisc}
\def\etal{{\textit{et~al.~}}}
\newtheorem{proposition}{Proposition}
\newtheorem{theorem}{Theorem}
\begin{document}
\newcolumntype{P}[1]{>{\centering\arraybackslash}p{#1}}

%
\title{On the Effectiveness of Least Squares Generative Adversarial Networks}
%
%
%
%

\author{Xudong~Mao,
        Qing~Li,~\IEEEmembership{Senior~Member,~IEEE},
        Haoran~Xie,~\IEEEmembership{Member,~IEEE},
        Raymond~Y.K.~Lau,~\IEEEmembership{Senior~Member,~IEEE},
        Zhen~Wang,
        and~Stephen~Paul~Smolley
\IEEEcompsocitemizethanks{
\IEEEcompsocthanksitem X. Mao and Q. Li are with  Department of Computer Science, City University of Hong Kong, Hong Kong.\protect\\
E-mail: xudong.xdmao@gmail.com, itqli@cityu.edu.hk
\IEEEcompsocthanksitem H. Xie is with Department of Mathematics and Information Technology, The Education University of Hong Kong, Hong Kong. \protect\\
E-mail: hrxie2@gmail.com
\IEEEcompsocthanksitem R. Lau is with Department of Information Systems, City University of Hong Kong, Hong Kong. E-mail: raylau@cityu.edu.hk
\IEEEcompsocthanksitem Z. Wang is with Center for Optical Imagery Analysis and Learning and School of Mechanical Engineering, Northwestern Polytechnical University, Xian 710072, China. E-mail: zhenwang0@gmail.com
\IEEEcompsocthanksitem S. Smolley is with CodeHatch Corp., Edmonton, Alberta Canada.\protect\\
E-mail: steve@codehatch.com
}}
\IEEEtitleabstractindextext{%
\begin{abstract}
Unsupervised learning with generative adversarial networks (GANs) has proven to be hugely successful. Regular GANs hypothesize the discriminator as a classifier with the sigmoid cross entropy loss function. However, we found that this loss function may lead to the vanishing gradients problem during the learning process. To overcome such a problem, we propose in this paper the Least Squares Generative Adversarial Networks (LSGANs) which adopt the least squares loss for both the discriminator and the generator. We show that minimizing the objective function of LSGAN yields minimizing the Pearson $\chi^2$ divergence. We also show that the derived objective function that yields minimizing the Pearson $\chi^2$ divergence performs better than the classical one of using least squares for classification. There are two benefits of LSGANs over regular GANs. First, LSGANs are able to generate higher quality images than regular GANs. Second, LSGANs perform more stably during the learning process. For evaluating the image quality, we conduct both qualitative and quantitative experiments, and the experimental results show that LSGANs can generate higher quality images than regular GANs. Furthermore, we evaluate the stability of LSGANs in two groups. One is to compare between LSGANs and regular GANs without gradient penalty. We conduct three experiments, including Gaussian mixture distribution, difficult architectures, and a newly proposed method --- datasets with small variability, to illustrate the stability of LSGANs. The other one is to compare between LSGANs with gradient penalty (LSGANs-GP) and WGANs with gradient penalty (WGANs-GP). The experimental results show that LSGANs-GP succeed in training for all the difficult architectures used in WGANs-GP, including 101-layer ResNet.
\end{abstract}

}

\maketitle

\IEEEdisplaynontitleabstractindextext

%
\IEEEpeerreviewmaketitle

\IEEEraisesectionheading{\section{Introduction}\label{sec:introduction}}

%
%
%
%
\IEEEPARstart{D}{eep} learning has launched a profound reformation and even been applied to many real-world tasks, such as image classification ~\cite{He2015}, object detection ~\cite{Ren2015}, and segmentation ~\cite{Long2014}. These tasks fall into the scope of supervised learning, which means that a lot of labeled data is provided for the learning processes. Compared with supervised learning, however, unsupervised learning (such as generative models) obtains limited impact from deep learning. Although some deep generative models, e.g., RBM ~\cite{Hinton2006}, DBM ~\cite{Salakhutdinov2009}, and VAE ~\cite{Kingma2013}, have been proposed, these models all face the difficulties of intractable functions (e.g., intractable partition function) or intractable inference, which in turn restricts the effectiveness of these models. 

Unlike the above deep generative models which usually adopt approximation methods for intractable functions or inference, Generative Adversarial Networks (GANs)~\cite{Goodfellow2014} requires no approximate inference and can be trained end-to-end through a differentiable network ~\cite{Goodfellow2016}. The basic idea of GANs is to train a discriminator and a generator simultaneously: the discriminator aims to distinguish between real samples and generated samples; while the generator tries to generate fake samples as real as possible, making the discriminator believe that the fake samples are from real data. GANs have demonstrated impressive performance for various computer vision tasks such as image generation ~\cite{Nguyen2016,Chen2016}, image super-resolution ~\cite{Ledig2016}, and semi-supervised learning ~\cite{Salimans2016}.

\begin{figure*}[t]
\centering
\begin{tabular}{ccc}

 \includegraphics[width=0.3\textwidth]{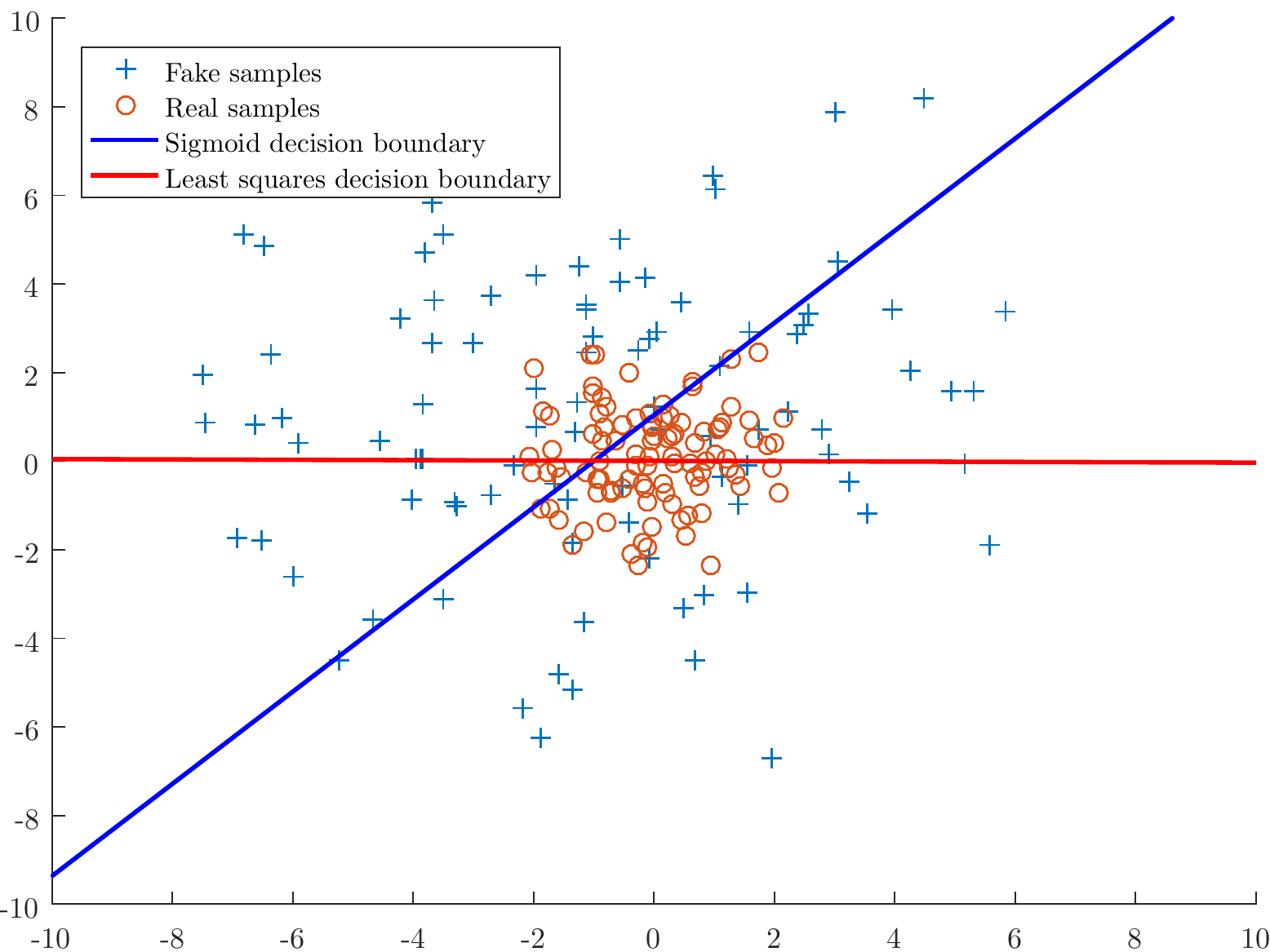}
 &
 \includegraphics[width=0.3\textwidth]{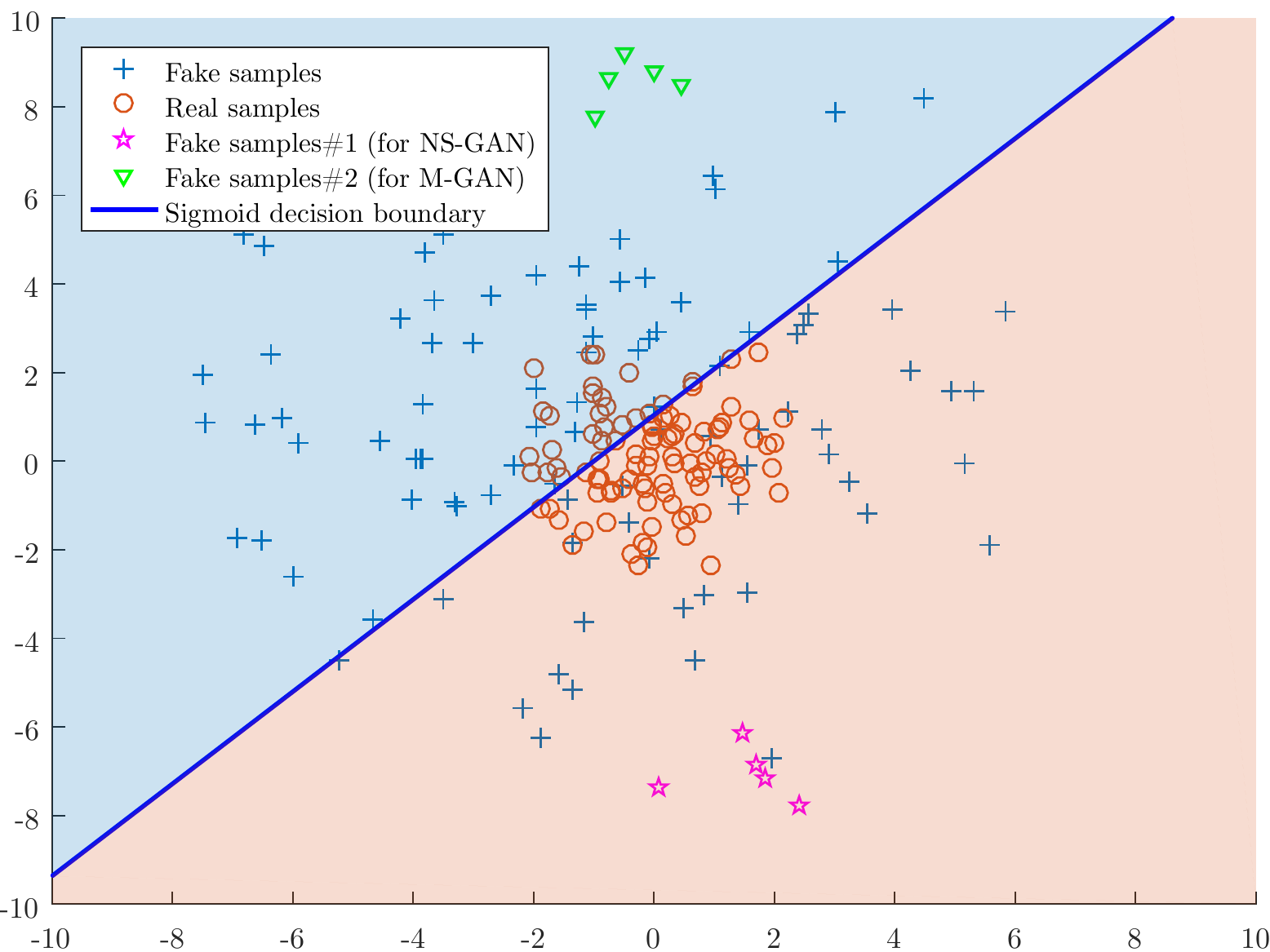}
  &
 \includegraphics[width=0.3\textwidth]{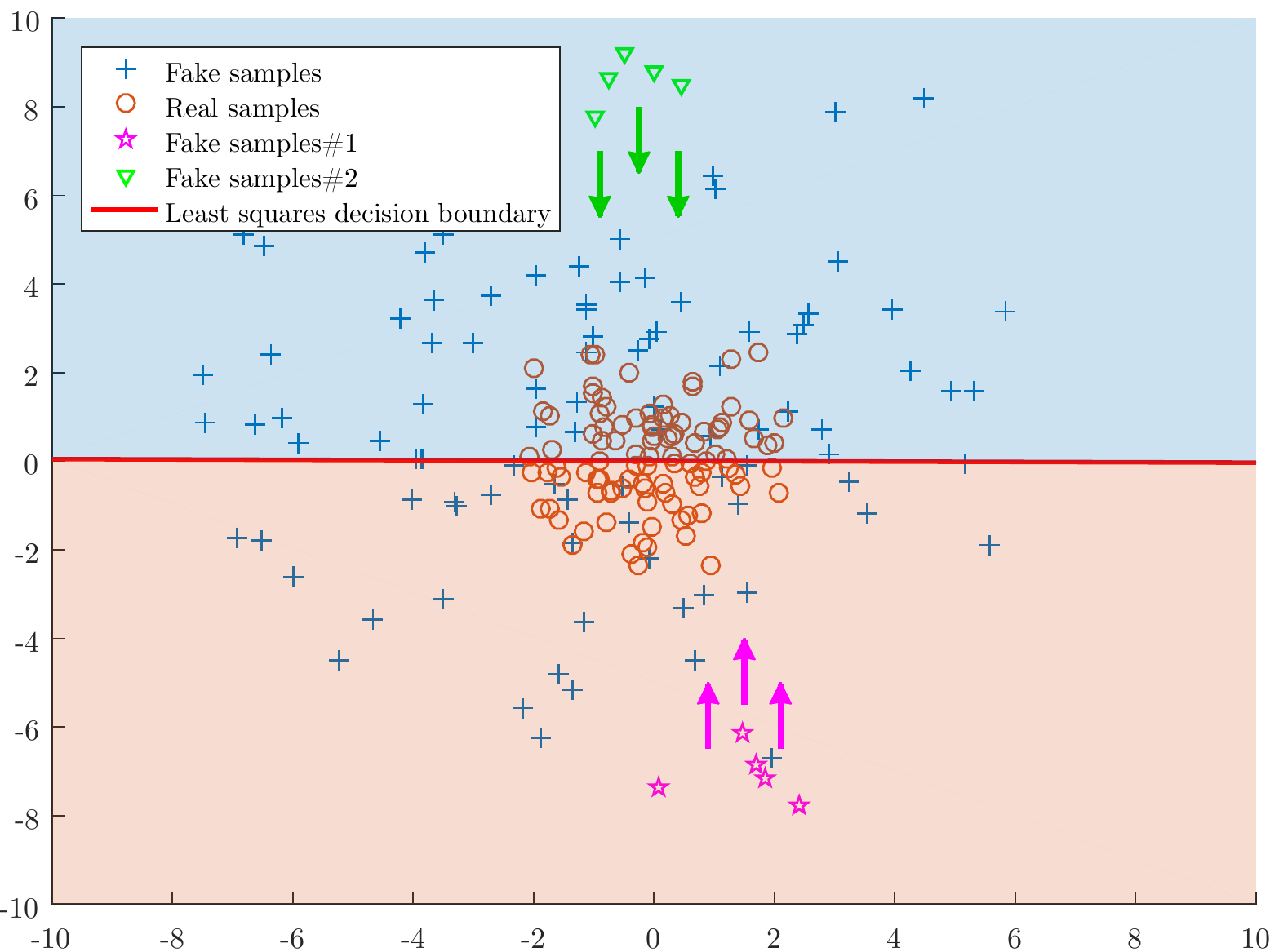}
\\
(a)
&
(b)
&
(c)
\end{tabular}
\caption{
Illustration of different behaviors of two loss functions. (a): Decision boundaries of two loss functions. Note that the decision boundary should go across real data distribution for a successful GANs learning. Otherwise, the learning process is saturated. (b): Decision boundary of the sigmoid cross entropy loss function. The orange area is the side of real samples, and the blue area is the side of fake samples. The non-saturating loss and the minimax loss will cause almost no gradient for the fake samples in magenta and green, respectively, when we use them to update the generator. (c): Decision boundary of the least squares loss function. It penalizes the fake samples (both in magenta and green), and as a result, it forces the generator to generate samples toward the decision boundary.
}
\label{fig:boundary}
\end{figure*}

\begin{figure}[t]
\centering
\begin{tabular}{cc}
 \includegraphics[width=1.5in]{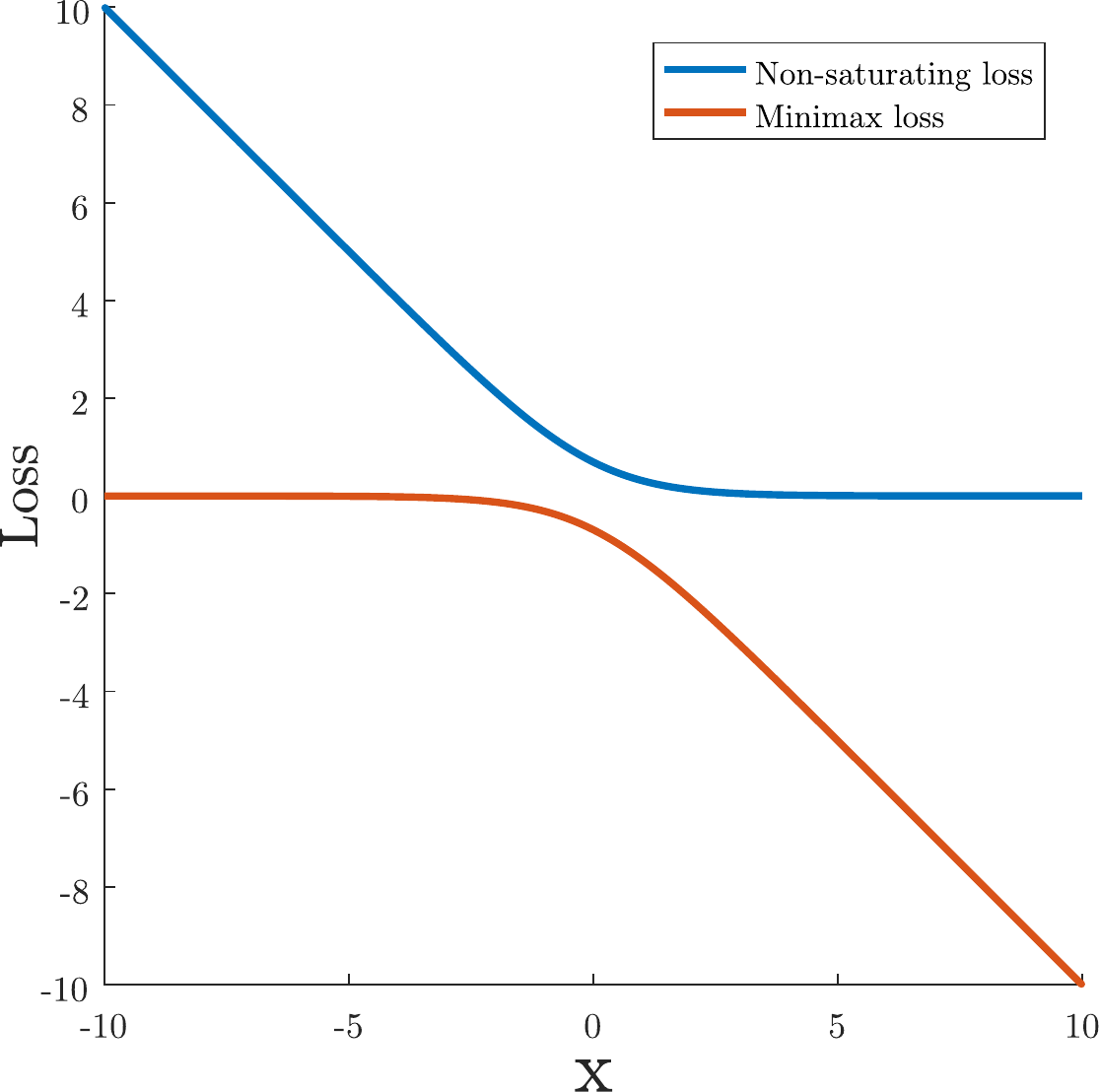}
&
 \includegraphics[width=1.5in]{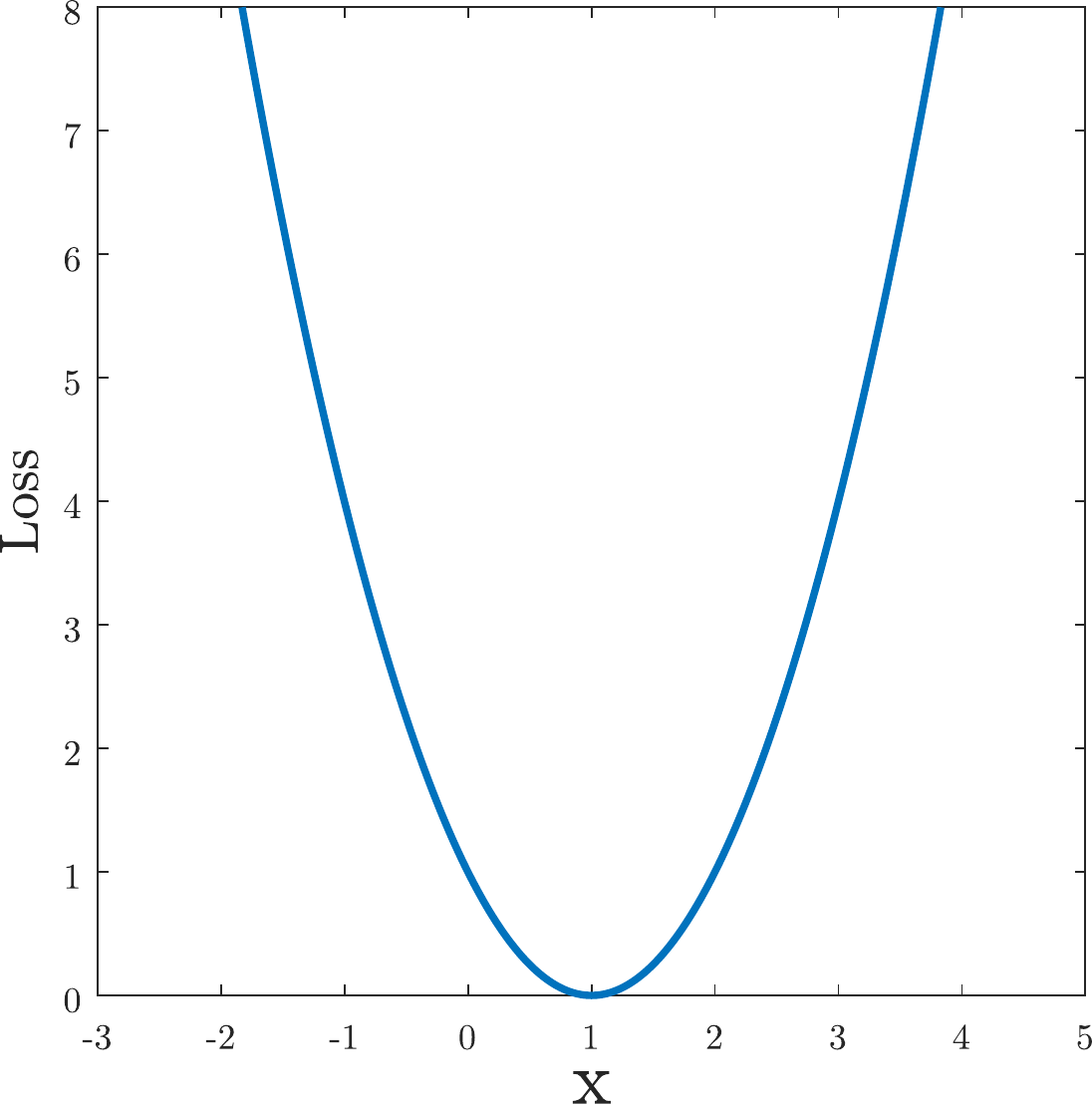}
\\
(a)
&
(b)
\end{tabular}
\caption{
(a): The non-saturating loss and the minimax loss. (b): The least squares loss.
}
\label{fig:loss}
\end{figure}

The original GAN paper ~\cite{Goodfellow2014} adopted the sigmoid cross entropy loss for the discriminator and presented two different losses for the generator: the ``minimax" loss (M-GANs) and the ``non-saturating" loss (NS-GANs). They have pointed out that M-GANs will saturate at the early stage of the learning process. Thus NS-GANs are recommended for use in practice. 

We argue that both the non-saturating loss and the minimax loss, however, will lead to the problem of vanishing gradients when updating the generator. As Fig. \ref{fig:loss}(a) shows, the non-saturating loss, i.e., \mbox{$-\log D(\cdot)$}, will saturate when the input is relatively large, while the minimax loss, i.e., \mbox{$\log (1 - D(\cdot)$)}, will saturate when the input is relatively small. Consequently, as Fig. \ref{fig:boundary}(b) shows, when updating the generator, the non-saturating loss will cause almost no gradient for the fake samples in magenta, because these samples are on the side of real data, corresponding to the input with relatively large values in Fig. \ref{fig:loss}(a). Similarly, the minimax loss will cause almost no gradient for the fake samples in green. However, these fake samples are still far from real data, and we want to pull them closer to real data. Based on this observation, we propose the Least Squares Generative Adversarial Networks (LSGANs) which adopt the least squares loss for both the discriminator and the generator. The idea is simple yet powerful: the least squares loss is able to move the fake samples toward the decision boundary, because the least squares loss penalizes samples that lie in a long way to the decision boundary even though they are on the correct side. As Fig. \ref{fig:boundary}(c) shows, the least squares loss will penalize the above two types of fake samples and pull them toward the decision boundary. Based on this property, LSGANs are able to generate samples that are closer to real data.

Another benefit of LSGANs is the improved training stability. Generally speaking, training GANs is a difficult issue in practice because of the instability of GANs learning ~\cite{Radford2015,Metz2016}. Recently, several papers have pointed out that the instability of GANs learning is partially caused by the objective function~\cite{Arjovsky2017, Metz2016,Qi2016}. Specifically, minimizing the objective function of regular GANs may cause the problem of vanishing gradients, which makes it hard to update the generator. LSGANs can alleviate this problem because penalizing samples based on the distances to the decision boundary can generate more gradients when updating the generator. Moreover, we theoretically show that the training instability of regular GANs is due to the mode-seeking behavior ~\cite{Bishop2006} of the objective function, while LSGANs exhibit less mode-seeking behavior.

In this paper, we also propose a new method for evaluating the stability of GANs. One popular evaluation method is to use difficult architectures, e.g., by excluding the batch normalization ~\cite{Arjovsky2017}. However, in practice, one will always select the stable architectures for their tasks. Sometimes the difficulty is from the datasets. Motivated by this, we propose to use difficult datasets but stable architectures to evaluate the stability of GANs. Specifically, we create two synthetic digit datasets with small variability by rendering $28\times 28$ digits using some standard fonts. Datasets with small variability are difficult for GANs to learn, since the discriminator can distinguish the real samples very easily for such datasets. 

Recently, gradient penalty has shown the effectiveness of improving the stability of GANs training ~\cite{Kodali2017,Gulrajani2017}. We find that gradient penalty is also helpful for improving the stability of LSGANs. By adding the gradient penalty in ~\cite{Kodali2017}, LSGANs are able to train successfully for all the difficult architectures used in WGANs-GP ~\cite{Gulrajani2017}. However, gradient penalty also has some inevitable disadvantages such as additional computational cost and memory cost. Based on this observation, we evaluate the stability of LSGANs in two settings: LSGANs without gradient penalty and LSGANs with gradient penalty.

Our contributions in this paper can be summarized as follows:
\begin{itemize}
\item We propose LSGANs which adopt least squares loss for both the discriminator and the generator. We show that minimizing the objective function of LSGANs yields minimizing the Pearson $\chi^2$ divergence.
\item We show that the derived objective function that yields minimizing the Pearson $\chi^2$ divergence performs better than the classical one of using least squares for classification.
\item We evaluate the image quality of LSGANs on several datasets including the LSUN-scenes and a cat dataset, and the experimental results demonstrate that LSGANs can generate higher quality images than NS-GANs. 
\item We also evaluate LSGANs on four datasets using the quantitative evaluation metric of Fr\'{e}chet inception distance (FID), and the results show that LSGANs with the derived objective function outperform NS-GANs on four datasets and outperform WGANs-GP on three datasets. Furthermore, LSGANs spend a quarter of the time comparing with WGANs-GP to reach a similar relatively optimal FID on LSUN-bedroom.
\item A new evaluation method for the training stability is proposed. We propose to use datasets with small variability to evaluate the stability of GANs. Furthermore, two synthetic digit datasets with small variability are created and published.
\item We evaluate the training stability of LSGANs without gradient penalty through three experiments including Gaussian mixture distribution, difficult architectures, and datasets with small variability. The experimental results demonstrate that LSGANs perform more stably than NS-GANs.
\item We also evaluate the training stability of LSGANs-GP through training on six difficult architectures used in WGANs-GP. LSGANs-GP succeed in training for all the six architectures including 101-layer ResNet.
\end{itemize}

This paper extends our earlier conference work ~\cite{Mao2017} in a number of ways. First, we present more theoretical analysis about the properties of LSGANs and $\chi^2$ divergence. Second, we conduct a new quantitative experiment based on the FID evaluation metric, and the results demonstrate that LSGANs perform better than NS-GANs and WGANs-GP. The results also show that the derived objective function (Eq. \eqref{eq:lsgan_peason}) that yields minimizing the Pearson $\chi^2$ divergence performs better than the classical one (Eq. \eqref{eq:lsgan_01}) of using least squares for classification.  Eq. \eqref{eq:lsgan_01} is used in our earlier conference work, but we change to use Eq. \eqref{eq:lsgan_peason} in this paper. Third, we provide new qualitative results on a cat dataset, which also shows that LSGANs generate higher quality images than NS-GANs. Fourth, we propose a new method for evaluating the training stability. In addition to using difficult architectures ~\cite{Arjovsky2017}, we propose to use difficult datasets but stable architectures to evaluate the training stability. Fifth, we present a new comparison experiment between LSGANs-GP and WGANs-GP. The results show that LSGANs-GP succeed in training for all the difficult architectures used in ~\cite{Gulrajani2017}, including 101-layer ResNet. Finally, we present a new comparison experiment between the two parameter schemes (Eq. \eqref{eq:lsgan_peason} and \eqref{eq:lsgan_01}) of LSGANs.

\section{Related Work}
\label{sec:related}

Deep generative models attempt to capture the probability distributions over the given data. Restricted Boltzmann Machines (RBMs), one type of deep generative models, are the basis of many other hierarchical models, and they have been used to model the distributions of images~\cite{Taylor2010} and documents~\cite{Hinton2009}. Deep Belief Networks (DBNs)~\cite{Hinton2006_DBN} and Deep Boltzmann Machines (DBMs)~\cite{Salakhutdinov2009} are extended from the RBMs. The most successful application of DBNs is for image classification~\cite{Hinton2006_DBN}, where DBNs are used to extract feature representations. However, RBMs, DBNs, and DBMs all have the difficulties of intractable partition functions or intractable posterior distributions, which thus use the approximation methods to learn the models. Another important deep generative model is Variational Autoencoders (VAE)~\cite{Kingma2013}, a directed model, which can be trained with gradient-based optimization methods. But VAEs are trained by maximizing the variational lower bound, which may lead to the blurry problem of generated images ~\cite{Goodfellow2016}.

Recently, Generative Adversarial Networks (GANs) have been proposed by Goodfellow \etal ~\cite{Goodfellow2014}, who explained the theory of GANs learning based on a game theoretic scenario. A similar idea is also introduced by Ganin \etal ~\cite{Ganin2016}, where a method of adversarial training is proposed for domain adaptation. Showing the powerful capability for unsupervised tasks, GANs have been applied to many specific tasks, like image super-resolution~\cite{Ledig2016}, text to image synthesis~\cite{Reed2016}, and image to image translation~\cite{Isola2016}. By combining the traditional content loss and the adversarial loss, super-resolution generative adversarial networks~\cite{Ledig2016} achieved state-of-the-art performance for the task of image super-resolution. Reed \etal ~\cite{Reed2016} proposed a model to synthesize images given text descriptions based on the conditional GANs~\cite{Mirza2014}. Isola \etal~\cite{Isola2016} also used the conditional GANs to transfer images from one representation to another. In addition to unsupervised learning tasks, GANs also show the good potential for semi-supervised learning tasks. Salimans \etal ~\cite{Salimans2016} proposed a GAN-based framework for semi-supervised learning, in which the discriminator not only outputs the probability that an input image is from real data, but also outputs the probabilities of belonging to each class. Another important problem of GANs is to inference the latent vectors from given examples ~\cite{Donahue2017,Dumoulin2017,Li2017}. Both ~\cite{Donahue2017} and ~\cite{Dumoulin2017} proposed a bidirectional adversarial learning framework by incorporating an encoder into the GANs framework. Li \etal ~\cite{Li2017} proposed to use the conditional entropy to regularize the objectives in ~\cite{Donahue2017,Dumoulin2017}, making the learning process more stable.

Despite the great successes GANs have achieved, improving the quality of generated images is still a challenge. A lot of works have been proposed to improve the quality of images for GANs. Radford \etal~\cite{Radford2015} first introduced convolutional layers to GANs architecture, and proposed a network architecture called deep convolutional generative adversarial networks (DCGANs). Denton \etal~\cite{Denton2015} proposed a framework called Laplacian pyramid of generative adversarial networks to improve the image quality of high-resolution images, where a Laplacian pyramid is constructed to generate high-resolution images starting from low-resolution images. A similar approach is proposed by Huang \etal ~\cite{Huang2016} who used a series of stacked GANs to generate images from abstract to specific. Salimans \etal ~\cite{Salimans2016} proposed a technique called feature matching to get better convergence. The idea is to make the generated samples match the statistics of real data by minimizing the mean square error on an intermediate layer of the discriminator.

Another critical issue for GANs is the stability of the learning process. Many works have been proposed to address this problem by analyzing the objective functions of GANs ~\cite{Arjovsky2017,Che2016,Metz2016,Nowozin2016,Qi2016}. Viewing the discriminator as an energy function, Zhao \etal ~\cite{Zhao2016} used an auto-encoder architecture to improve the stability of GANs learning. Dai \etal ~\cite{Dai2017} extended the energy-based GANs by adding some regularizations to make the discriminator non-degenerate. To make the generator and the discriminator more balanced, Metz \etal~\cite{Metz2016} created an unrolled objective function to enhance the generator. Che \etal~\cite{Che2016} incorporated a reconstruction module and used the distance between real samples and reconstructed samples as a regularizer to get more stable gradients. Nowozin \etal~\cite{Nowozin2016} pointed out that the objective of regular GAN~\cite{Goodfellow2014} which is related to Jensen-Shannon divergence is a special case of divergence estimation, and generalized it to arbitrary f-divergences~\cite{Nguyen2010}. Arjovsky \etal~\cite{Arjovsky2017} extended this by analyzing the properties of four different divergences and concluded that Wasserstein distance is more stable than Jensen-Shannon divergence. Qi ~\cite{Qi2016} proposed the loss-sensitive GAN whose loss function is based on the assumption that real samples should have smaller losses than fake samples. They also introduced to use Lipschitz regularity to stabilize the learning process. Base on the above assumptions, they proved that loss-sensitive GAN has non-vanishing gradient almost everywhere. Some other techniques to stabilize GANs learning include the second order method ~\cite{Mescheder2017} and gradient penalty ~\cite{Gulrajani2017,Kodali2017,Roth2017}. Mescheder \etal ~\cite{Mescheder2017} analyzed the convergence property of GANs from the perspective of the eigenvalues of the equilibrium and proposed a method to regularize the eigenvalues, which in turn leads to better training stability. Gulrajani \etal ~\cite{Gulrajani2017} used gradient penalty to enforce the Lipschitz constraint in Wasserstein distance. They showed that this approach performs more stably than the method used in ~\cite{Arjovsky2017}. Unlike ~\cite{Gulrajani2017} that applies gradient penalty around the region between the real data and the fake data, Kodali \etal ~\cite{Kodali2017} proposed to apply gradient penalty around the real data manifold only, which has the advantage that it is applicable to various GANs. Roth \etal ~\cite{Roth2017} derived a new gradient-based regularization from analyzing that adding noise to the discriminator yields training with gradient penalty.

\section{Method}
\label{sec:method}

\subsection{Generative Adversarial Networks} 
The learning process of GANs is to train a discriminator $D$ and a generator $G$ simultaneously. The target of $G$ is to learn the distribution $p_g$ over data $\bm{x}$. $G$ starts with sampling input variables $\bm{z}$ from a uniform or Gaussian distribution $p_z(\bm{z})$, then maps the input variables $\bm{z}$ to data space $G(\bm{z}; \theta_g)$ through a differentiable network. On the other hand, $D$ is a classifier $D(\bm{x}; \theta_d)$ that aims to recognize whether an image is from training data or from $G$. The minimax objective for GANs can be formulated as follows:

\begin{equation}
\label{eq:gan}
\begin{split}
\min_G \max_D V_{\text{\tiny GAN}}(D, G) = \mathbb{E}_{\bm{x} \sim p_{\text{data}}(\bm{x})}[\log D(\bm{x})]& \\
+ \mathbb{E}_{\bm{z} \sim p_{\bm{z}}(\bm{z})}[\log (1 - D(G(\bm{z})))]&.
\end{split}
\end{equation}
In practice, Goodfellow \etal ~\cite{Goodfellow2014} recommend implementing the following non-saturating loss for the generator, which provides much stronger gradients.

\begin{equation}
\label{eq:nsgan}
\begin{split}
\min_G V_{\text{\tiny GAN}}(G) = -\mathbb{E}_{\bm{z} \sim p_{\bm{z}}(\bm{z})}[\log (D(G(\bm{z})))].
\end{split}
\end{equation}
Following ~\cite{Fedus2018}, we refer to Eq. \eqref{eq:gan} as minimax GANs (M-GANs) and Eq. \eqref{eq:nsgan} as non-saturating GANs (NS-GANs). In the following experiments, we compare our proposed LSGANs with NS-GANs since NS-GANs perform much better than M-GANs ~\cite{Goodfellow2014,Fedus2018}. 

\subsection{Least Squares Generative Adversarial Networks}
\label{sec:lsgan}
As stated in Section \ref{sec:introduction}, the original GAN paper ~\cite{Goodfellow2014} adopted the sigmoid cross entropy loss function for the discriminator, and introduced the minimax loss and the non-saturating loss for the generator. However, both the minimax loss and the non-saturating loss will cause the problem of vanishing gradients for some fake samples that are far from real data, as shown in Fig. \ref{fig:boundary}(b). To remedy this problem, we propose the Least Squares Generative Adversarial Networks (LSGANs). Suppose we use the $a$-$b$ coding scheme for the discriminator, where $a$ and $b$ are the labels for the fake data and the real data, respectively. Then the objective functions for LSGANs can be defined as follows:

\begin{equation}
\label{eq:lsgan}
\begin{split}
\min_D V_{\text{\tiny LSGAN}}(D) = &\frac{1}{2}\mathbb{E}_{\bm{x} \sim p_{\text{data}}(\bm{x})}\bigl[(D(\bm{x})-b)^2\bigr] \\
+ &\frac{1}{2}\mathbb{E}_{\bm{z} \sim p_{\bm{z}}(\bm{z})}\bigl[(D(G(\bm{z}))-a)^2\bigr] \\
\min_G V_{\text{\tiny LSGAN}}(G) = &\frac{1}{2}\mathbb{E}_{\bm{z} \sim p_{\bm{z}}(\bm{z})}\bigl[(D(G(\bm{z}))-c)^2\bigr],
\end{split}
\end{equation}
where $c$ denotes the value that $G$ wants $D$ to believe for the fake data.

\subsubsection{Benefits of LSGANs}\label{sec:benefits}
The benefits of LSGANs can be derived from two aspects. First, unlike M-GANs and NS-GANs which cause almost no gradient for some kinds of fake samples, LSGANs will penalize those samples even though they are correctly classified, as shown in Fig. \ref{fig:boundary}(c). When we update the generator, the parameters of the discriminator are fixed, i.e., the decision boundary is fixed. As a result, the penalization will cause the generator to generate samples toward the decision boundary. On the other hand, the decision boundary should go across the manifold of real data for a successful GANs learning; otherwise, the learning process will be saturated. Thus moving the generated samples toward the decision boundary leads to making them closer to the manifold of real data.

Second, penalizing the samples lying in a long way to the decision boundary can generate more gradients when updating the generator, which in turn relieves the problem of vanishing gradients. This allows LSGANs to perform more stably during the learning process. This benefit can also be derived from another perspective: as shown in Fig. \ref{fig:loss}, the least squares loss function is flat only at one point, while NS-GANs will saturate when $x$ is relatively large, and M-GANs will saturate when $x$ is relatively small. In Section \ref{sec:benefit_chi}, we provide further theoretical analysis about the stability of LSGANs.

\subsection{Theoretical Analysis}
\subsubsection{Relation to $\chi^2$ Divergence}
In the original GAN paper~\cite{Goodfellow2014}, the authors have shown that minimizing Eq. \eqref{eq:gan} yields minimizing the Jensen-Shannon divergence:
\footnotesize
\begin{equation}
\label{eq:gan_js}
\begin{split}
C(G) &=\text{KL} \left(p_\text{data} \left \| \frac{p_\text{data} + p_g}{2} \right. \right) + \text{KL} \left(p_g \left \| \frac{p_\text{data} + p_g}{2} \right. \right)-\log(4). 
\end{split}
\end{equation}
\normalsize

Here we also explore the relation between LSGANs and f-divergence. Consider the following extension of Eq. \eqref{eq:lsgan}:
\begin{equation}
\label{eq:general_lsgan}
\begin{split}
\min_D V_{\text{\tiny LSGAN}}(D) = &\frac{1}{2}\mathbb{E}_{\bm{x} \sim p_{\text{data}}(\bm{x})}\bigl[(D(\bm{x})-b)^2\bigr] \\
+ &\frac{1}{2}\mathbb{E}_{\bm{z} \sim p_{\bm{z}}(\bm{z})}\bigl[(D(G(\bm{z}))-a)^2\bigr] \\
\min_G V_{\text{\tiny LSGAN}}(G) = &\frac{1}{2}\mathbb{E}_{\bm{x} \sim p_{\text{data}}(\bm{x})}\bigl[(D(\bm{x})-c)^2\bigr] \\
+ &\frac{1}{2}\mathbb{E}_{\bm{z} \sim p_{\bm{z}}(\bm{z})}\bigl[(D(G(\bm{z}))-c)^2\bigr].
\end{split}
\end{equation}

\begin{figure*}[t]
\centering
\begin{tabular}{c}
 \includegraphics[width=0.9\textwidth]{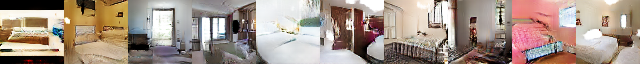}
\\
(a) Generated images ($64 \times 64$) by NS-GANs (reported in ~\cite{Radford2015}).
\\
 \includegraphics[width=0.9\textwidth]{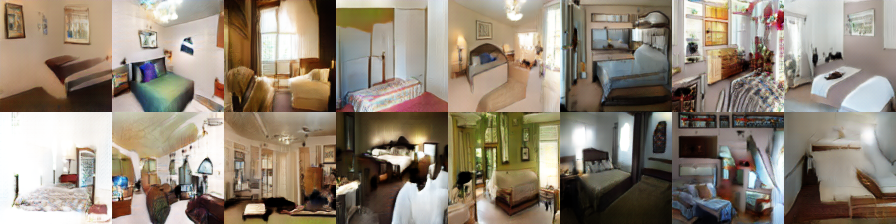}
\\
(b) Generated images ($112 \times 112$) by NS-GANs.
\\
 \includegraphics[width=0.9\textwidth]{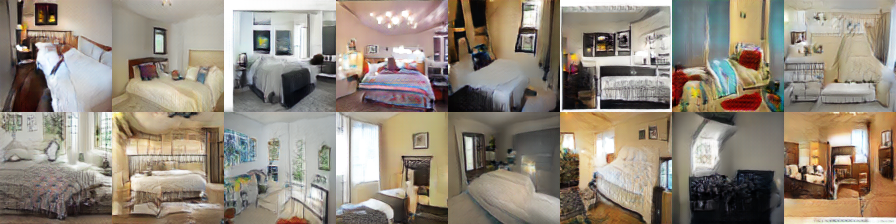}
\\
(c) Generated images ($112 \times 112$) by LSGANs.

\end{tabular}
\caption{
Generated images on LSUN-bedroom.
}
\label{fig:bedroom_cmp}
\end{figure*}

Note that adding the term $\mathbb{E}_{\bm{x} \sim p_{\text{data}}(\bm{x})}[(D(\bm{x})-c)^2]$ to $V_{\text{\tiny LSGAN}}(G)$ causes no change of the optimal values since this term does not contain parameters of $G$.

We first derive the optimal discriminator $D$ for a fixed $G$.
\begin{proposition}
\label{pro:optimal_d}
For a fixed $G$, the optimal discriminator $D$ is
\begin{equation}
\label{eq:optimal_d}
D^*(\bm{x}) = \frac{bp_\text{data}(\bm{x})+ap_g(\bm{x})}{p_\text{data}(\bm{x})+p_g(\bm{x})}.
\end{equation}
\end{proposition}

\begin{proof}

Given any generator $G$, we try to minimize $V(D)$ with respect to the discriminator $D$:
\small
\begin{equation}
\label{eq:optimal_proof}
\begin{split}
V(D) = &\frac{1}{2}\mathbb{E}_{\bm{x} \sim p_{\text{data}}}\bigl[(D(\bm{x})-b)^2\bigr] + \frac{1}{2}\mathbb{E}_{\bm{z} \sim p_{\bm{z}}}\bigl[(D(G(\bm{z}))-a)^2\bigr] \\
= &\frac{1}{2}\mathbb{E}_{\bm{x} \sim p_{\text{data}}}\bigl[(D(\bm{x})-b)^2\bigr] + \frac{1}{2}\mathbb{E}_{\bm{x} \sim p_{g}}\bigl[(D(\bm{x})-a)^2\bigr] \\
=&\int_{\mathcal{X}} \frac{1}{2} \bigl( p_{\text{data}}(\bm{x})(D(\bm{x})-b)^2 + p_{g}(\bm{x})(D(\bm{x})-a)^2 \bigr) \textrm{d}x.
\end{split}
\end{equation}
\normalsize 
Consider the internal function:
\begin{equation}
\frac{1}{2} \bigl( p_{\text{data}}(\bm{x})(D(\bm{x})-b)^2 + p_{g}(\bm{x})(D(\bm{x})-a)^2 \bigr).
\end{equation}
It achieves the mimimum at $\frac{bp_\text{data}(\bm{x})+ap_g(\bm{x})}{p_\text{data}(\bm{x})+p_g(\bm{x})}$ with respect to $D(\bm{x})$, concluding the proof.

\end{proof}

In the following equations we use $p_\text{d}$ to denote $p_\text{data}$ for simplicity. 

\begin{theorem}
Optimizing LSGANs yields minimizing the Pearson $\chi^2$ divergence between $p_\text{d}+p_g$ and $2p_g$, if a, b, and c satisfy the conditions of $b-c=1$ and $b-a=2$ in Eq. \eqref{eq:general_lsgan}. 
\end{theorem}

\begin{proof}
We can reformulate $V_{\text{\tiny LSGAN}}(G)$ in Eq. \eqref{eq:general_lsgan} by using Proposition \ref{pro:optimal_d}:

\small
\begin{equation}
\label{eq:lsgan_divergence}
\begin{split}
2C(G) &= \mathbb{E}_{\bm{x} \sim p_{\text{d}}}\bigl[(D^*(\bm{x})-c)^2\bigr]+\mathbb{E}_{\bm{z} \sim p_{z}}\bigl[(D^*(G(\bm{z}))-c)^2\bigr] \\
&=\mathbb{E}_{\bm{x} \sim p_{\text{d}}}\bigl[(D^*(\bm{x})-c)^2\bigr]+\mathbb{E}_{\bm{x} \sim p_{g}}\bigl[(D^*(\bm{x})-c)^2\bigr] \\
&=\mathbb{E}_{\bm{x} \sim p_{\text{d}}}
\left[
\bigl(\frac{bp_\text{d}(\bm{x})+ap_g(\bm{x})}{p_\text{d}(\bm{x})+p_g(\bm{x})}-c\bigr)^2
\right] \\
&+\mathbb{E}_{\bm{x} \sim p_{g}}
\left[
\bigl(\frac{bp_\text{d}(\bm{x})+ap_g(\bm{x})}{p_\text{d}(\bm{x})+p_g(\bm{x})}-c\bigr)^2
\right] \\
&=\int_{\mathcal{X}}p_\text{d}(\bm{x}) \bigl(\frac{(b-c)p_\text{d}(\bm{x})+(a-c)p_g(\bm{x})}{p_\text{d}(\bm{x})+p_g(\bm{x})}\bigr)^2 \textrm{d}\bm{x} \\
&+ \int_{\mathcal{X}}p_g(\bm{x}) \bigl(\frac{(b-c)p_\text{d}(\bm{x})+(a-c)p_g(\bm{x})}{p_\text{d}(\bm{x})+p_g(\bm{x})}\bigr)^2 \textrm{d}\bm{x} \\
&=\int_{\mathcal{X}} \frac{\bigl((b-c)p_\text{d}(\bm{x})+(a-c)p_g(\bm{x})\bigr)^2}{p_\text{d}(\bm{x})+p_g(\bm{x})} \textrm{d}\bm{x} \\
&=\int_{\mathcal{X}} \frac{\bigl((b-c)(p_\text{d}(\bm{x})+p_g(\bm{x}))-(b-a)p_g(\bm{x})\bigr)^2}{p_\text{d}(\bm{x})+p_g(\bm{x})} \textrm{d}\bm{x}.
\end{split}
\end{equation}
\normalsize 
If we set $b-c=1$ and $b-a=2$, then
\begin{equation}
\label{eq:lsgan_divergence_final}
\begin{split}
2C(G)&=\int_{\mathcal{X}} \frac{\bigl(2p_g(\bm{x})-(p_\text{d}(\bm{x})+p_g(\bm{x}))\bigr)^2}{p_\text{d}(\bm{x})+p_g(\bm{x})} \textrm{d}\bm{x} \\
&=\chi^2_\text{Pearson}(p_\text{d}+p_g\|2p_g),
\end{split}
\end{equation}
where $\chi^2_\text{Pearson}$ is the Pearson $\chi^2$ divergence. Thus minimizing Eq. \eqref{eq:general_lsgan} yields minimizing the Pearson $\chi^2$ divergence between $p_\text{d}+p_g$ and $2p_g$ if $a$, $b$, and $c$ satisfy the conditions of $b-c=1$ and $b-a=2$. 

\end{proof}

\begin{figure*}[t]
\centering
\begin{tabular}{cc}
 \includegraphics[width=0.45\textwidth]{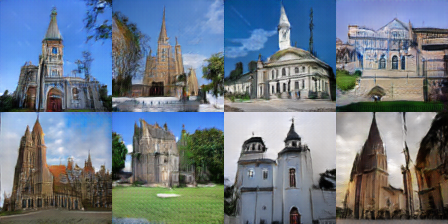}
&
 \includegraphics[width=0.45\textwidth]{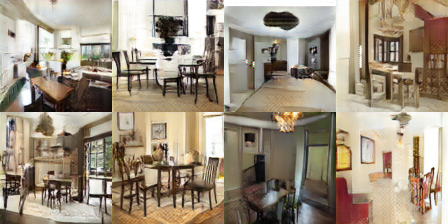}
\\
(a) \small Church outdoor.
&
(b) \small Dining room.
\\
 \includegraphics[width=0.45\textwidth]{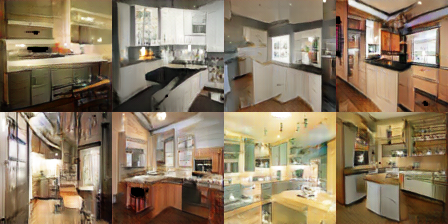}
&
 \includegraphics[width=0.45\textwidth]{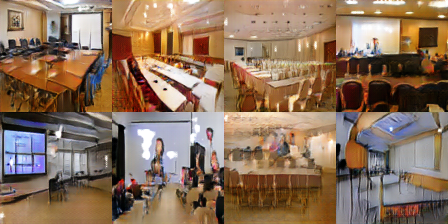}
\\
(c) \small Kitchen.
&
(d) \small Conference room.
\end{tabular}
\caption{
Generated images on different scene datasets.
}
\label{fig:scene}
\end{figure*}

\subsubsection{Properties of $\chi^2$ Divergence}
\label{sec:benefit_chi}
As Eq. \eqref{eq:gan_js} shows, the original GAN has been proven to optimize the JS divergence. Furthermore, Husz{\'a}r ~\cite{Ferenc2015} pointed out that Eq. \eqref{eq:gan_js} can be viewed as an interpolation between $\text{KL}(p_g \| p_d)$ and $\text{KL}(p_d \| p_g)$:

\begin{equation}
\label{eq:js_kl}
\begin{split}
\text{JS}_{\pi}(p_d\|p_g) &=(1-\pi) \text{KL}(p_d\|\pi p_d+(1-\pi)p_g) \\
&+ \pi \text{KL}(p_g\|\pi p_d+(1-\pi)p_g),
\end{split}
\end{equation}
where Eq. \eqref{eq:gan_js} corresponds to $\pi=0.5$. They also found that optimizing Eq. \eqref{eq:gan_js} tends to perform similarly to $\text{KL}(p_g\|p_d)$. $\text{KL}(p_g\|p_d)$ is widely used in variational inference due to the convenient evidence lower bound ~\cite{Bishop2006}. However, optimizing $\text{KL}(p_g\|p_d)$ has the problem of mode-seeking behavior or under-dispersed approximations ~\cite{Bishop2006,Dieng2017,Ferenc2015}. This problem also appears in GANs learning, which is known as the mode collapse problem. The definition of $\text{KL}(p_g\|p_d)$ is given below:

\begin{equation}
\label{eq:kl}
\begin{split}
\text{KL}(p_g\|p_d) = -\int_{\mathcal{X}} p_g(\bm{x})\ln\left(\frac{p_d(\bm{x})}{p_g(\bm{x})}\right)\text{d}\bm{x}. 
\end{split}
\end{equation}
The mode-seeking behavior of $\text{KL}(p_g\|p_d)$ can be understood by noting that $p_g$ will be close to zero where $p_d$ is near zero, because $\text{KL}(p_g\|p_d)$ will be infinite if $p_d=0$ and $p_g>0$. This is called the zero-forcing property ~\cite{Bishop2006}.

Recently, $\chi^2$ divergence has drawn researchers' attention in variational inference since $\chi^2$ divergence is able to produce over-dispersed approximations ~\cite{Dieng2017}. For the objective function in Eq. \eqref{eq:lsgan_divergence_final}, it will become infinite if $p_d+p_g=0$ and $p_g-p_d>0$, which will not happen since $p_g\geq0$ and $p_d\geq0$. Thus $\chi^2_\text{Pearson}(p_\text{d}+p_g\|2p_g)$ has no zero-forcing property. This makes LSGANs less mode-seeking and alleviates the mode collapse problem.

\subsection{Parameters Selection}
\label{sec:para}
One method to determine the values of $a$, $b$, and $c$ in Eq. \eqref{eq:lsgan} is to satisfy the conditions of $b-c=1$ and $b-a=2$, such that minimizing Eq. \eqref{eq:lsgan} yields minimizing the Pearson $\chi^2$ divergence between $p_\text{d}+p_g$ and $2p_g$. For example, by setting $a=-1$, $b=1$, and $c=0$, we get the following objective functions:
\begin{equation}
\label{eq:lsgan_peason}
\begin{split}
\min_D V_{\text{\tiny LSGAN}}(D) = &\frac{1}{2}\mathbb{E}_{\bm{x} \sim p_{\text{data}}(\bm{x})}\bigl[(D(\bm{x})-1)^2\bigr] \\
+ &\frac{1}{2}\mathbb{E}_{\bm{z} \sim p_{\bm{z}}(\bm{z})}\bigl[(D(G(\bm{z}))+1)^2\bigr] \\
\min_G V_{\text{\tiny LSGAN}}(G) = &\frac{1}{2}\mathbb{E}_{\bm{z} \sim p_{\bm{z}}(\bm{z})}\bigl[(D(G(\bm{z})))^2\bigr].
\end{split}
\end{equation}

Another method is to make $G$ generate samples as real as possible by setting $c=b$, corresponding to the traditional way of using least squares for classification. For example, by using the $0$-$1$ binary coding scheme, we get the following objective functions:
\begin{equation}
\label{eq:lsgan_01}
\begin{split}
\min_D V_{\text{\tiny LSGAN}}(D) = &\frac{1}{2}\mathbb{E}_{\bm{x} \sim p_{\text{data}}(\bm{x})}\bigl[(D(\bm{x})-1)^2\bigr] \\
+ &\frac{1}{2}\mathbb{E}_{\bm{z} \sim p_{\bm{z}}(\bm{z})}\bigl[(D(G(\bm{z})))^2\bigr] \\
\min_G V_{\text{\tiny LSGAN}}(G) = &\frac{1}{2}\mathbb{E}_{\bm{z} \sim p_{\bm{z}}(\bm{z})}\bigl[(D(G(\bm{z}))-1)^2\bigr].
\end{split}
\end{equation}

In practice, we find that Eq. \eqref{eq:lsgan_peason} shows better FID results and faster convergence speed than Eq. \eqref{eq:lsgan_01}, as demonstrated by experiments. For the experiments presented in our earlier conference work ~\cite{Mao2017}, we adopted Eq. \eqref{eq:lsgan_01}, but for the newly introduced experiments in this paper, Eq. \eqref{eq:lsgan_peason} is adopted.

\begin{figure*}[t]
\centering
\begin{tabular}{cccccc}  
\multicolumn{6}{c}{\includegraphics[width=0.91\textwidth]{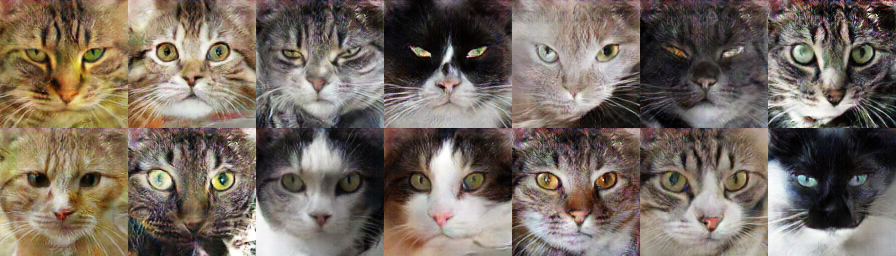}}
\\
 \multicolumn{6}{c}{(a) Generated cats ($128 \times 128$) by NS-GANs.}
\\
 \multicolumn{6}{c}{\includegraphics[width=0.91\textwidth]{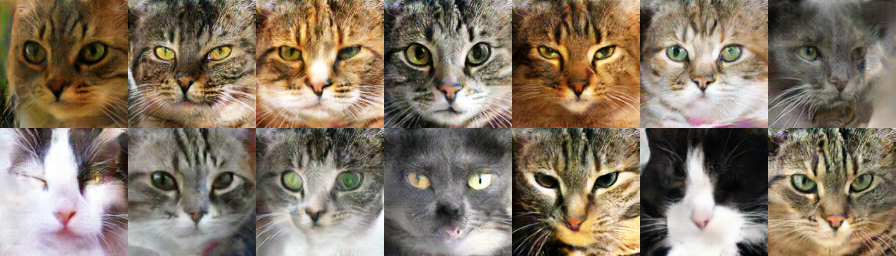}}
\\
 \multicolumn{6}{c}{(b) Generated cats ($128 \times 128$) by LSGANs.}
\\

 \includegraphics[width=0.13\textwidth]{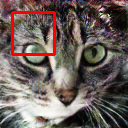} & \includegraphics[width=0.13\textwidth]{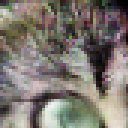} & \includegraphics[width=0.13\textwidth]{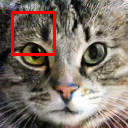} & \includegraphics[width=0.13\textwidth]{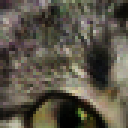} & \includegraphics[width=0.13\textwidth]{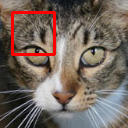} & \includegraphics[width=0.13\textwidth]{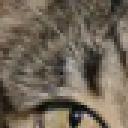} 
\\
\multicolumn{2}{c}{(c) NS-GANs} & \multicolumn{2}{c}{(d) LSGANs} & \multicolumn{2}{c}{(e) Real Sample}
\end{tabular}
\caption{
Generated images on cat datasets. (c)(d)(e): Comparison by zooming in on the details of the images. LSGANs generate cats with sharper and more exquisite hair and faces than the ones generated by NS-GANs. 
}
\label{fig:cat}
\end{figure*}

\section{Experiments}
\label{sec:experiments}
In this section, we first present some details of our implementation. Next, we present the results of the qualitative evaluation and quantitative evaluation of LSGANs. Then we evaluate the stability of LSGANs in two groups. One is to compare LSGANs with DCGANs without gradient penalty by three experiments. The other one is to compare LSGANs-GP with WGANs-GP. Note that we implement DCGANs using the non-saturating loss (NS-GANs). In the following experiments, we denote NS-GANs as the baseline method.

\subsection{Implementation Details}
The implementation of our proposed models is based on a public implementation of DCGANs\footnote{https://github.com/carpedm20/DCGAN-tensorflow} using TensorFlow ~\cite{tensorflow2015}. The learning rate is set to $0.0002$ except for LSUN-scenes whose learning rate is set to $0.001$. The mini-batch size is set to 64, and the variables are initialized from a Gaussian distribution with a mean of zero and a standard deviation of 0.02. Following DCGANs, $\beta_1$ for Adam optimizer is set to 0.5. The pixel values of all the images are scaled to [-1,1], since we use the Tanh in the generator to produce images. Our implementation is available at https://github.com/xudonmao/improved\_LSGAN.

\begin{figure*}[t]
\centering
\begin{tabular}{c}
 \includegraphics[width=0.96\textwidth]{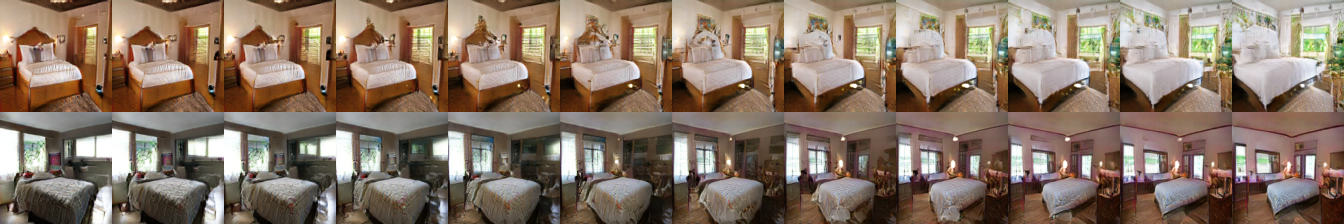}
\\
(a) Interpolation on the LSUN-bedroom dataset.
\\
 \includegraphics[width=0.96\textwidth]{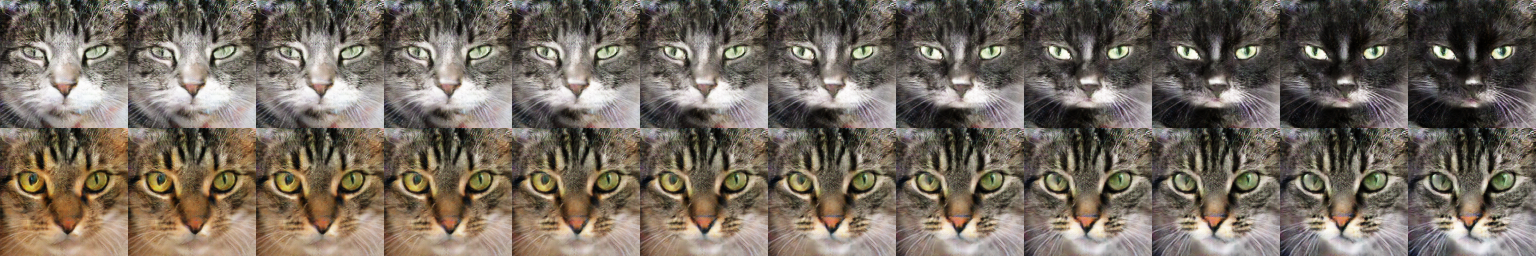}
\\
(b) Interpolation on the cats dataset.
\end{tabular}
\caption{
Interpolation result by LSGANs. The generated images show smooth transitions, which indicates that LSGANs have learned semantic representations in the latent space.
}
\label{fig:interpolation}
\end{figure*}

\subsection{Image Quality}
\subsubsection{Qualitative Evaluation}
\vspace{2pt}
\noindent \textbf{Scenes Generation}
We train LSGANs and NS-GANs using the same network architecture on the LSUN-bedroom dataset. The network architecture is presented in Table \ref{tab:scene}. All the images are resized to the resolution of $112 \times 112$. The generated images by the two models are presented in Fig. \ref{fig:bedroom_cmp}. Compared with the images generated by NS-GANs, the texture details (e.g., the textures of beds) of the images generated by LSGANs are more exquisite, and the images generated by LSGANs look sharper. We also train LSGANs on four other scene datasets including church, dining room, kitchen, and conference room. The generated results are shown in Fig. \ref{fig:scene}.

\begin{table}[t]
\caption{The network architecture for scene generation, where CONV denotes the convolutional layer, TCONV denotes the transposed convolutional layer, FC denotes the fully-connected layer, BN denotes the batch normalization, LReLU denotes the Leaky-ReLU, and (K3,S2,O256) denotes a layer with $3 \times 3$ kernel, stride 2, and 256 output filters.}
\label{tab:scene}
\begin{tabular}{|c|c|}    
\hline
Generator & Discriminator \\ \hline
Input z & Input $112\times 112\times 3$\\ 
FC(O12544), BN, ReLU           & CONV(K5,S2,O64),      LReLU\\
TCONV(K3,S2,O256), BN, ReLU    & CONV(K5,S2,O128), BN, LReLU\\
TCONV(K3,S1,O256), BN, ReLU    & CONV(K5,S2,O256), BN, LReLU\\
TCONV(K3,S2,O256), BN, ReLU    & CONV(K5,S2,O512), BN, LReLU\\
TCONV(K3,S1,O256), BN, ReLU    & FC(O1)\\
TCONV(K3,S2,O128), BN, ReLU    & Loss\\
TCONV(K3,S2,O64), BN, ReLU    & \\
TCONV(K3,S1,O3), Tanh    & \\
\hline
\end{tabular}
\end{table}

\begin{table}[t]
\caption{The network architecture for cats generation. The meanings of the symbols can be found in Table \ref{tab:scene}.}
\label{tab:cat}
\begin{tabular}{|c|c|}    
\hline
Generator & Discriminator \\ \hline
Input z & Input $128\times 128\times 3$\\ 
FC(O32768), BN, ReLU           & CONV(K5,S2,O64),      LReLU\\
TCONV(K3,S2,O256), BN, ReLU    & CONV(K5,S2,O128), BN, LReLU\\
TCONV(K3,S2,O128), BN, ReLU    & CONV(K5,S2,O256), BN, LReLU\\
TCONV(K3,S2,O64), BN, ReLU     & CONV(K5,S2,O512), BN, LReLU\\
TCONV(K3,S2,O3), Tanh          & FC(O1)\\
                               & Loss\\
\hline
\end{tabular}
\end{table}

\vspace{2pt}
\noindent \textbf{Cats Generation}
We further evaluate LSGANs on a cat dataset ~\cite{Zhang2008}. We first use the preprocess methods in a public project\footnote{https://github.com/AlexiaJM/Deep-learning-with-cats} to get cat head images whose resolution is larger than $128 \times 128$, and then resize all the images to the resolution of $128 \times 128$. The network architecture used in this task is presented in \mbox{Table \ref{tab:cat}}. We use the following evaluation protocol for comparing the performance between LSGANs and NS-GANs. First, we train LSGANs and NS-GANs using the same architecture on the cat dataset. During training, we save a checkpoint of the model and a batch of generated images every $1000$ iterations. Second, we select the best models of LSGANs and NS-GANs by checking the quality of saved images in every $1000$ iterations. Finally, we use the selected best models to randomly generate cat images and compare the quality of generated images. The selected models of LSGANs and NS-GANs are available at https://github.com/xudonmao/improved\_LSGAN. Fig. \ref{fig:cat} shows the generated cat images of LSGANs and NS-GANs. We observe that LSGANs generate cats with sharper and more exquisite hair than the ones generated by NS-GANs. Fig. \ref{fig:cat}(c)(d) shows the details of the cat hair by zooming in the generated images. We observe that the cat hair generated by NS-GANs contains more artificial noise. By checking more generated samples using the above saved models, we also observe that the overall quality of generated images by LSGANs is better than NS-GANs.

\vspace{2pt}
\noindent \textbf{Walking in the Latent Space} We also present the interpolation results in Fig. \ref{fig:interpolation}. The result of walking in the latent space is a sign of whether a model is just memorizing the training dataset. We first randomly sample two points of the noise vector $\bm{z}$, and then interpolate the vector values between the two sampled points. The images in Fig. \ref{fig:interpolation} show smooth transitions, which indicates that LSGANs have learned semantic representations in the latent space.

\subsubsection{Quantitative Evaluation}
\label{sec:quantitative}
For the quantitative evaluation of LSGANs, we adopt Fr\'{e}chet Inception Distance (FID) ~\cite{Heusel2017} as the evaluation metric. FID measures the distance between the generated images and the real images by approximating the feature space of the inception model as a multidimensional Gaussian distribution, which has been proved to be more consistent with human judgment than inception score ~\cite{Salimans2016}. Smaller FID values mean closer distances between the generated and real images. We also conduct a human subjective study on the LSUN-bedroom dataset.

\begin{table}[t]
\renewcommand{\arraystretch}{1.1}
\caption{FID results of NS-GANs, WGANs-GP, and LSGANs on four datasets. $\text{LSGANs}_{(-110)}$ and $\text{LSGANs}_{(011)}$ refer to Eq. \eqref{eq:lsgan_peason} and Eq. \eqref{eq:lsgan_01}, respectively.}
\label{tab:fid}
\normalsize
\begin{tabular}{@{\hspace{0.0\tabcolsep}}P{2.5cm}@{\hspace{0.7\tabcolsep}}P{1.1cm}P{1.1cm}P{1.1cm}P{1.1cm}@{\hspace{1.8\tabcolsep}}}
\hline
Method & LSUN         &Cat& ImageNet &  CIFAR10 \\
\hline
NS-GANs & $28.04$     &$15.81$ & $74.15$   &$35.25$\\
WGANs-GP & $22.77$    &$29.03$  & $\textbf{62.05}$     &$40.83$\\
$\text{LSGANs}_{(011)}$ & $27.21$    &$15.46$  & $72.54$     &$36.46$\\
$\text{LSGANs}_{(-110)}$& $\textbf{21.55}$ &$\textbf{14.28}$ & $68.95$ & $\textbf{35.19}$\\
\hline
\end{tabular}
\centering
\end{table}

\begin{figure}[t]
\centering
\begin{tabular}{cc}
 \includegraphics[width=1.6in]{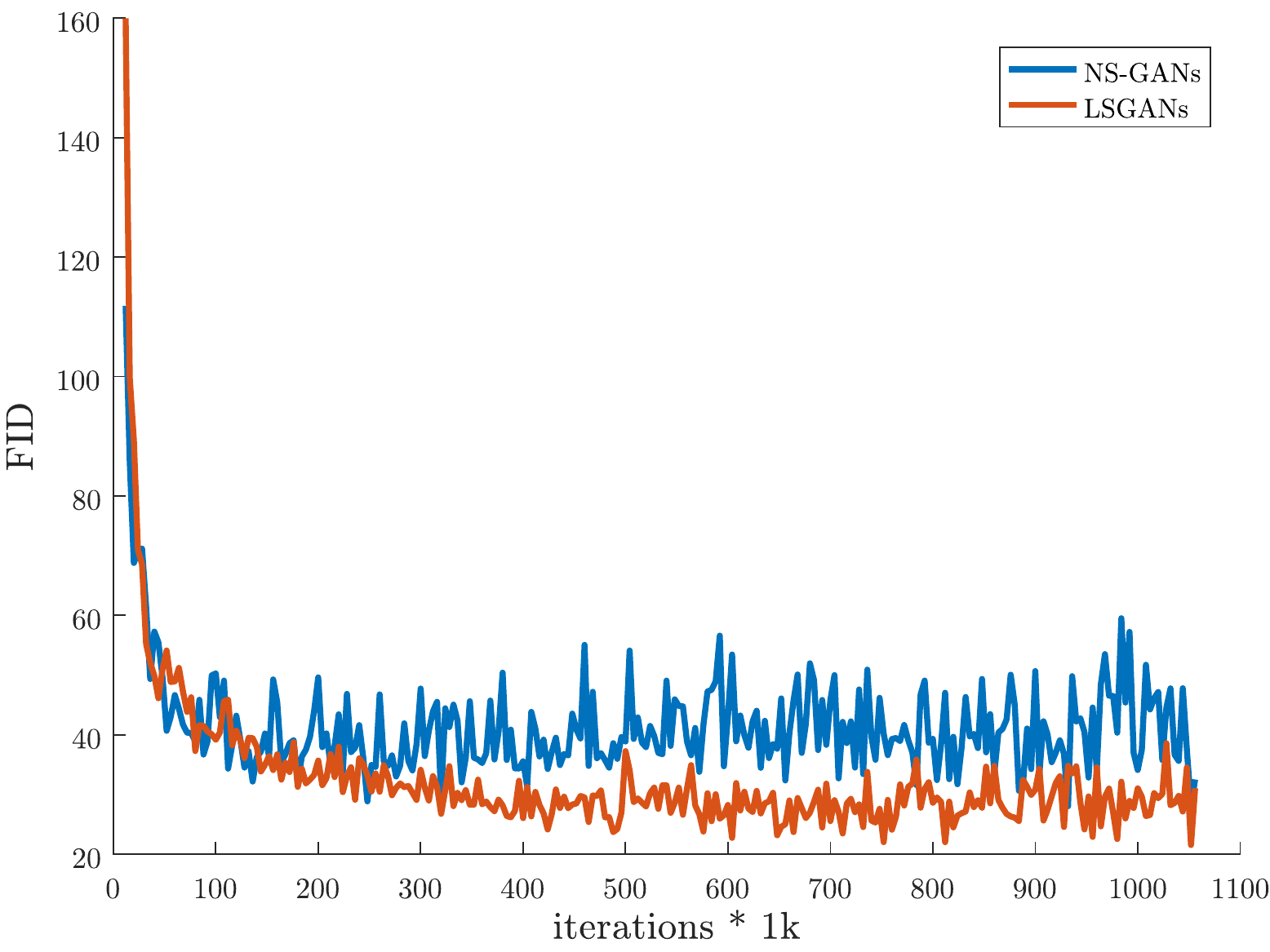}
&
 \includegraphics[width=1.6in]{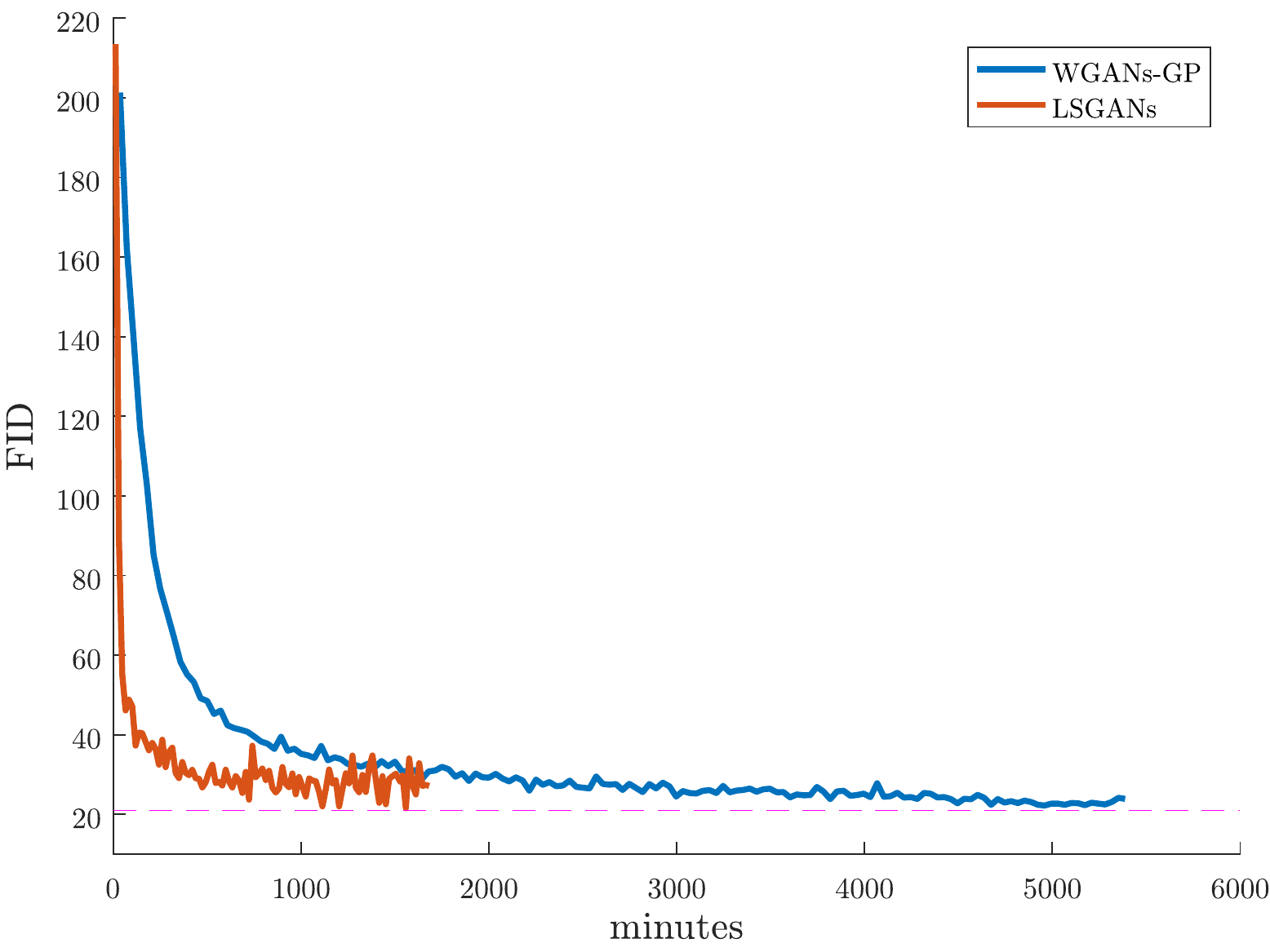}
\\
(a)
&
(b)
\end{tabular}
\caption{
(a): Comparison of FID on LSUN between NS-GANs and LSGANs during the learning process, which is aligned with iterations. (b): Comparison of FID on LSUN between WGANs-GP and LSGANs during the learning process, which is aligned with wall-clock time. 
}
\label{fig:fid}
\end{figure}

 \begin{figure*}[t]
\centering
 \includegraphics[width=0.9\textwidth]{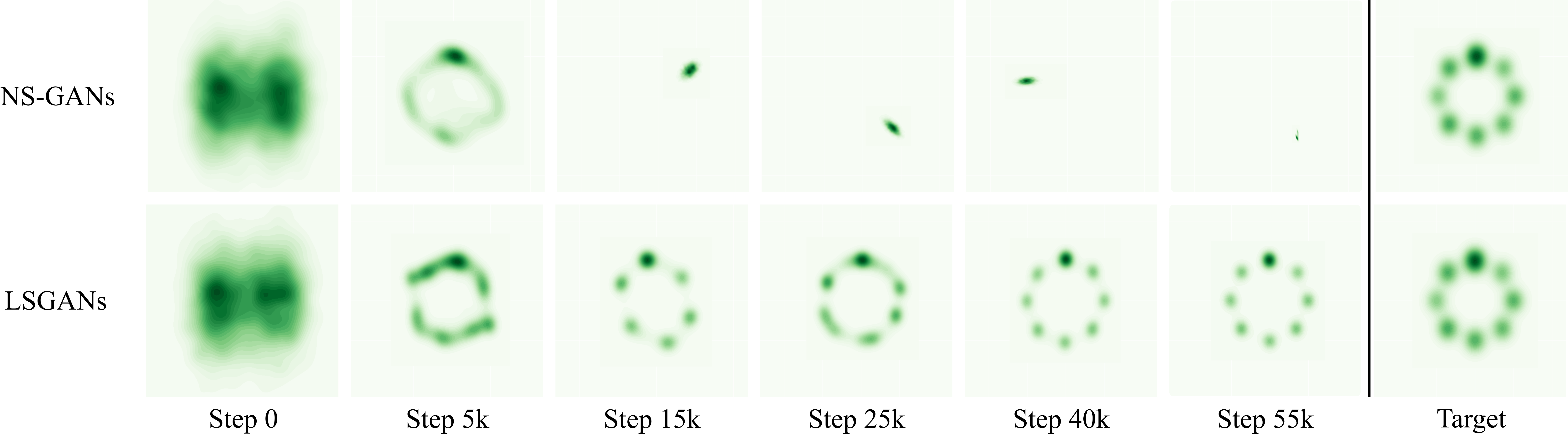}
\caption{
 Dynamic results of Gaussian kernel estimation for NS-GANs and LSGANs. The final column shows the distribution of real data. 
}
\label{fig:gaussian}
\end{figure*}

\vspace{2pt}
\noindent \textbf{Fr\'{e}chet Inception Distance} For FID, we evaluate the performances of LSGANs, NS-GANs, and WGANs-GP on several datasets including LSUN-bedroom, the cat dataset, ImageNet, and CIFAR-10. We also compare the performances of Eq. \eqref{eq:lsgan_peason} (denoted as $\text{LSGANs}_{(-110)}$) and Eq. \eqref{eq:lsgan_01} (denoted as $\text{LSGANs}_{(011)}$). For a fair comparison, all the models are trained with the same architecture proposed in DCGAN ~\cite{Radford2015} (i.e., four convolutional layers for both the discriminator and the generator), and the dimension of the noise input is set to $100$. For WGANs-GP, we adopt the official implementation for evaluation. The resolutions for LSUN, Cat, ImageNet, and CIFAR-10 are $64\times64$, $128\times128$, $64\times64$, and $32\times32$, respectively. We randomly generate $50,000$ images every 4k iterations for each model and then compute FID. The results are shown in Table \ref{tab:fid}, and we have the following four observations. First, $\text{LSGANs}_{(-110)}$ outperform NS-GANs for all the four datasets. Second, comparing with WGANs-GP, $\text{LSGANs}_{(-110)}$ perform better for three datasets, especially for the cat dataset. Third, $\text{LSGANs}_{(-110)}$ perform better than $\text{LSGANs}_{(011)}$ for all the four datasets. Fourth, the performance of $\text{LSGANs}_{(011)}$ is comparable to NS-GANs.

We also show the FID plot of the learning process in Fig. \ref{fig:fid}, where LSGANs refer to $\text{LSGANs}_{(-110)}$. Following ~\cite{Heusel2017}, the plots of NS-GANs and LSGANs are aligned by iterations, and the plots of WGANs-GP and LSGANs are aligned by wall-clock time. As Fig. \ref{fig:fid}(a) shows, NS-GANs and LSGANs show similar FID at the first 25k iterations, but LSGANs can decrease FID after 25k iterations, achieving better performance eventually. Fig. \ref{fig:fid}(b) shows that WGANs-GP and LSGANs achieve similar optimal FID eventually, but LSGANs spend much less time ($1,100$ minutes) than WGANs-GP ($4,600$ minutes) to reach a relatively optimal FID around $22$. This is due to that WGANs-GP need multiple updates for the discriminator and need additional computational time for the gradient penalty.

\vspace{2pt}
\noindent \textbf{Human Subjective Study}
To further evaluate the performance of LSGANs, we conduct a human subjective study using the generated bedroom images ($112 \times 112$) from NS-GANs and LSGANs with the same network architecture. We randomly construct image pairs, where one image is from NS-GANs and the other one is from LSGANs. We ask Amazon Mechanical Turk annotators to judge which image looks more realistic. With 4,000 votes totally, NS-GANs get 43.6\% votes and LSGANs get 56.4\% votes, i.e., an overall 12.8\% increase of votes over NS-GANs.

\subsection{Training Stability}
In this section, we evaluate the stability of our proposed LSGANs and compare with two baselines including NS-GANs and WGANs-GP. Gradient penalty has been proven to be effective for improving the stability of GANs training ~\cite{Kodali2017,Gulrajani2017}, but it also has some inevitable disadvantages such as additional computational cost and memory cost. Thus we evaluate the stability of LSGANs in two groups. One is to compare with the model without gradient penalty (i.e., NS-GANs), and the other one is to compare with the model with gradient penalty (i.e., WGANs-GP). 

\subsubsection{Evaluation without Gradient Penalty}
\label{sec:stability_wo_gp}
We first compare LSGANs with NS-GANs, both of which are without gradient penalty. Three comparison experiments are conducted: 1) learning on a Gaussian mixture distribution; 2) learning with difficult architectures; and 3) learning on datasets with small variability.

\begin{table}[t]
\renewcommand{\arraystretch}{1.3}
\caption{Experiments on Gaussian mixture distribution. We run $100$ times for each model and record how many times that a model ever generates samples around one or two modes during the training process.}
\label{tab:gaussian}
\normalsize
\begin{tabular}{c@{\hskip 0.3in}c}
\hline
Method & \begin{tabular}{@{}c@{}}The number of generating samples \\around one or two modes \end{tabular}\\
\hline
NS-GANs&$99$ / $100$\\
LSGANs (ours)&$5$ / $100$\\
\hline
\end{tabular}
\centering
\end{table}

\vspace{2pt}
\noindent \textbf{Gaussian Mixture Distribution}
Learning on a Gaussian mixture distribution to evaluate the stability is proposed by Metz \etal ~\cite{Metz2016}. If the model suffers from the mode collapse problem, it will generate samples only around one or two modes. We train NS-GANs and LSGANs with the same network architecture on a 2D mixture of eight Gaussian mixture distribution, where both the generator and the discriminator contain three fully-connected layers. Fig. \ref{fig:gaussian} shows the dynamic results of Gaussian kernel density estimation. We can see that NS-GANs suffer from mode collapse starting at $15$k iterations. They only generate samples around a single valid mode of the data distribution. But LSGANs can learn the Gaussian mixture distribution successfully. We also try different architectures (four or five fully-connected layers) and different values of the hyper-parameters (the learning rate and the dimension of the noise vector). The results also show that NS-GANs tend to generate samples around one or two modes, while LSGANs are less prone to this problem.

To further verify the robustness of the above observation, we run $100$ times for each model and record how many times that a model suffers from the mode collapse problem. For each experiment, we save the density estimation every $5$k iterations and observe whether a model generates samples only around one or two modes in each saved estimation. The results show that NS-GANs appear to generate one or two modes 99 times out of 100, while LSGANs only have 5 times, as shown in Table \ref{tab:gaussian}.

\begin{figure*}[t]
\centering
\begin{tabular}{cccc}
NS-GANs & LSGANs & NS-GANs& LSGANs \\
\Xhline{2\arrayrulewidth}
\\
 \includegraphics[width=0.22\textwidth]{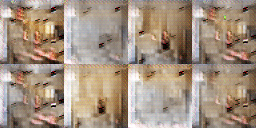}
&
 \includegraphics[width=0.22\textwidth]{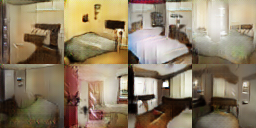}
&
 \includegraphics[width=0.22\textwidth]{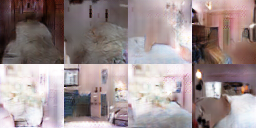}
&
 \includegraphics[width=0.22\textwidth]{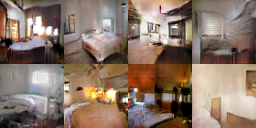}
\\
\multicolumn{2}{c}{(a) No BN in $G$ using Adam.}
&
\multicolumn{2}{c}{(b) No BN in either $G$ or $D$ using RMSProp.}
\\
 \includegraphics[width=0.22\textwidth]{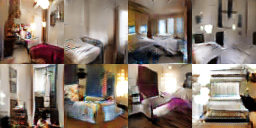}
&
 \includegraphics[width=0.22\textwidth]{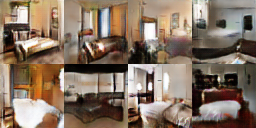}
&
 \includegraphics[width=0.22\textwidth]{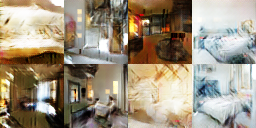}
&
 \includegraphics[width=0.22\textwidth]{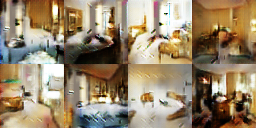}
\\
\multicolumn{2}{c}{(c) No BN in $G$ using RMSProp.}
&
\multicolumn{2}{c}{(d) No BN in either $G$ or $D$ using Adam.}

\end{tabular}
\caption{
Comparison experiments between NS-GANs and LSGANs by excluding batch normalization (BN).
}
\label{fig:no_BN}
\end{figure*}

\begin{figure*}[t]
\centering
\begin{tabular}{c@{\hskip 0.5in}c@{\hskip 0.5in}c}
Real Samples & NS-GANs & LSGANs \\
\Xhline{2\arrayrulewidth}
\\
 \includegraphics[width=0.22\textwidth]{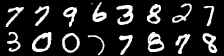}
&
 \includegraphics[width=0.22\textwidth]{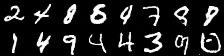}
&
 \includegraphics[width=0.22\textwidth]{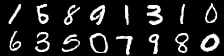}
\\
\multicolumn{3}{c}{(a) Training on MNIST.}
\\
 \includegraphics[width=0.22\textwidth]{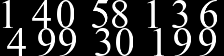}
&
 \includegraphics[width=0.22\textwidth]{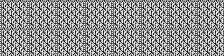}
&
 \includegraphics[width=0.22\textwidth]{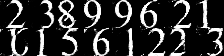}
\\
\multicolumn{3}{c}{(b) Training on a synthetic digit dataset with random horizontal shift.}
\\
 \includegraphics[width=0.22\textwidth]{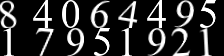}
&
 \includegraphics[width=0.22\textwidth]{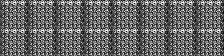}
&
 \includegraphics[width=0.22\textwidth]{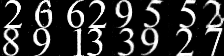}
\\
\multicolumn{3}{c}{(c) Training on a synthetic digit dataset with random horizontal shift and rotation.}
\end{tabular}
\caption{
Evaluation on datasets with small variability. All the tasks are conducted using the same network architecture as shown in Table \ref{tab:small_variance}. DCGANs succeed in learning on MNIST but fail on the two synthetic digit datasets with small variability, while LSGANs succeed in learning on all the three datasets.
}
\label{fig:small_variance}
\end{figure*}

\vspace{2pt}
\noindent \textbf{Difficult Architectures}
Another experiment is to train GANs with difficult architectures, which is proposed in ~\cite{Arjovsky2017}. The model will generate very similar images if it suffers from mode collapse problem. The network architecture used in this task is similar to the one in Table \ref{tab:cat} except for the image resolution. Based on this network architecture, two architectures are designed to compare the stability. The first one is to exclude the batch normalization in the generator ($\text{BN}_G$ for short), and the second one is to exclude the batch normalization in both the generator and discriminator ($\text{BN}_{GD}$ for short). As pointed out in ~\cite{Arjovsky2017}, the selection of optimizer is critical to the model performance. Thus we evaluate the two architectures with two optimizers, Adam~\cite{Kingma2014} and RMSProp~\cite{Tieleman2012}. In summary, we have the following four training settings: (1) $\text{BN}_G$ with Adam, (2) $\text{BN}_G$ with RMSProp, (3) $\text{BN}_{GD}$ with Adam, and (4) $\text{BN}_{GD}$ with RMSProp. We train the above models on the LSUN-bedroom dataset using NS-GANs and LSGANs separately. The results are shown in Fig. \ref{fig:no_BN}, and we make the following three major observations. First, for $\text{BN}_G$ with Adam, there is a chance for LSGANs to generate relatively good quality images. We test $10$ times, and $5$ of those succeed to generate relatively good quality images. For NS-GANs, however, we never observe successful learning, suffering from a severe degree of mode collapse. Second, for $\text{BN}_{GD}$ with RMSProp, as Fig. \ref{fig:no_BN} shows, LSGANs generate higher quality images than NS-GANs which have a slight degree of mode collapse. Third, LSGANs and NS-GANs have similar performance for $\text{BN}_G$ with RMSProp and $\text{BN}_{GD}$ with Adam. Specifically, for $\text{BN}_G$ with RMSProp, both LSGANs and NS-GANs can generate relatively good images. For $\text{BN}_{GD}$ with Adam, both have a slight degree of mode collapse.

\begin{table}[t]
\caption{The network architecture for stability evaluation on datasets with small variability. The meanings of the symbols can be found in Table \ref{tab:scene}.}
\label{tab:small_variance}
\begin{tabular}{|c|c|}    
\hline
Generator & Discriminator \\ \hline
Input z & Input $28\times 28\times 1$\\ 
FC(O8192), BN, ReLU            & CONV(K5,S2,O20),     LReLU\\
TCONV(K3,S2,O256), BN, ReLU    & CONV(K5,S2,O50), BN, LReLU\\
TCONV(K3,S2,O128), BN, ReLU    & FC(O500), BN, LReLU\\
TCONV(K3,S2,O1), Tanh          & FC(O1)\\
                               & Loss\\
\hline
\end{tabular}
\end{table}

\begin{figure*}[t]
\centering
\begin{tabular}{cccc}
WGANs-GP & LSGANs-GP & WGANs-GP & LSGANs-GP \\
\Xhline{2\arrayrulewidth}
\\
 \includegraphics[width=0.22\textwidth]{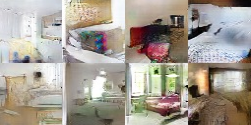}
&
 \includegraphics[width=0.22\textwidth]{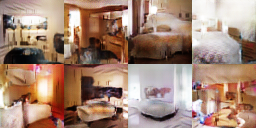}
&
 \includegraphics[width=0.22\textwidth]{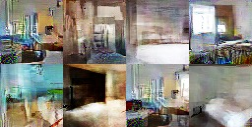}
&
 \includegraphics[width=0.22\textwidth]{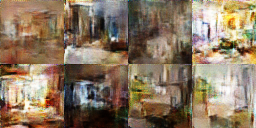}
\\
\multicolumn{2}{c}{(a) G: No BN and a constant number of filters, D: DCGAN.}
&
\multicolumn{2}{c}{(b) G: 4-layer 512-dim ReLU MLP, D: DCGAN.}
\\
 \includegraphics[width=0.22\textwidth]{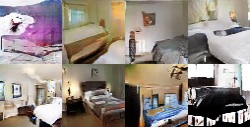}
&
 \includegraphics[width=0.22\textwidth]{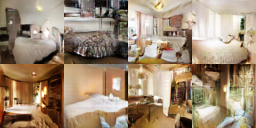}
&
 \includegraphics[width=0.22\textwidth]{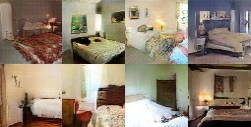}
&
 \includegraphics[width=0.22\textwidth]{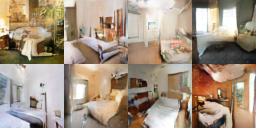}
\\
\multicolumn{2}{c}{(c) No normalization in either G or D.}
&
\multicolumn{2}{c}{(d) Gated multiplicative nonlinearities everywhere in G and D.}
\\
 \includegraphics[width=0.22\textwidth]{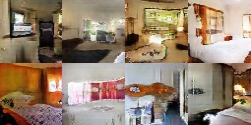}
&
 \includegraphics[width=0.22\textwidth]{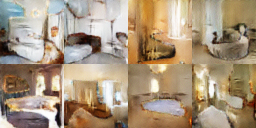}
&
 \includegraphics[width=0.22\textwidth]{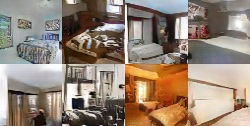}
&
 \includegraphics[width=0.22\textwidth]{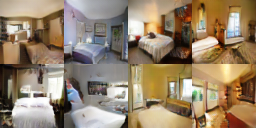}
\\
\multicolumn{2}{c}{(e) Tanh nonlinearities everywhere in G and D.}
&
\multicolumn{2}{c}{(f) 101-layer ResNet G and D.}
\end{tabular}
\caption{
Comparison experiments between WGANs-GP and LSGANs-GP using difficult architectures, where the images generated by WGANs-GP are duplicated from ~\cite{Gulrajani2017}. LSGANs-GP succeed for all the architectures.
}
\label{fig:cmp_wgan}
\end{figure*}

\begin{figure*}[t]
\centering
 \includegraphics[width=0.9\textwidth]{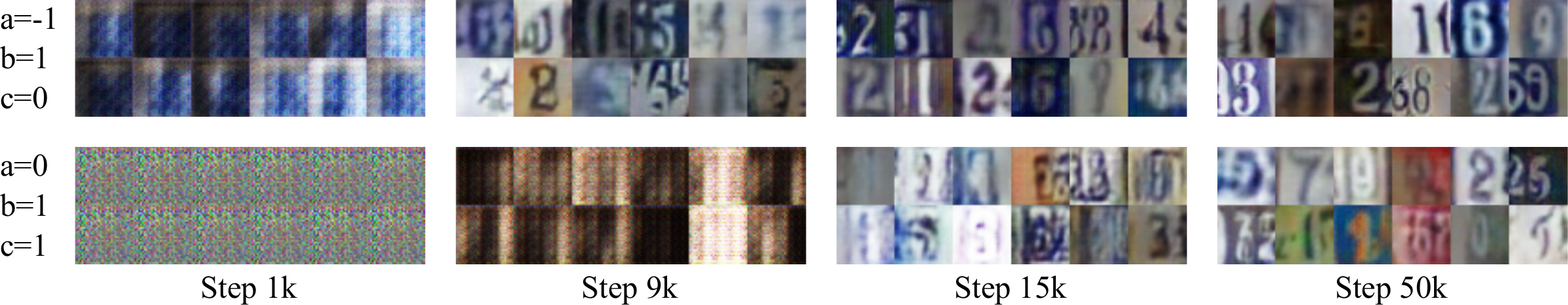}
\caption{
 Dynamic results of the two parameter schemes on the SVHN dataset. The first row corresponds to Eq. \eqref{eq:lsgan_peason}, and the second row corresponds to Eq. \eqref{eq:lsgan_01}.
}
\label{fig:svhn}
\end{figure*}

\vspace{2pt}
\noindent \textbf{Datasets with small variability}
Using difficult architectures is an effective way to evaluate the stability of GANs ~\cite{Arjovsky2017}. However, in practice, it is natural to select a stable architecture for a given task. The difficulty of a practical task is the task itself. Inspired by this motivation, we propose to use difficult datasets but stable architectures to evaluate the stability of GANs. We find that the datasets with small variability are difficult for GANs to learn, since the discriminator can distinguish the real samples very easily for the datasets with small variability. Specifically, we construct the datasets by rendering $28\times 28$ digits using the Times-New-Roman font. Two datasets are created\footnote{Available at https://github.com/xudonmao/improved\_LSGAN}: 1) one is applied with random horizontal shift; and 2) the other one is applied with random horizontal shift and random rotation from $0$ to $10$ degrees. Each category contains one thousand samples for both datasets. Note that the second dataset is with larger variability than the first one. Examples of the two synthetic datasets are shown in the first column of Fig. \ref{fig:small_variance}. We adopt a stable architecture for digits generation, following the suggestions in ~\cite{Radford2015}, where the discriminator is similar to LeNet, and the generator contains three transposed convolutional layers. The detail of the network architecture is presented in Table \ref{tab:small_variance}. We train NS-GANs and LSGANs on the above two datasets, and the generated images are shown in Fig. \ref{fig:small_variance}, along with the results on MNIST. We have two major observations. First, NS-GANs succeed in learning on MNIST but fail on the two synthetic digit datasets, while LSGANs succeed in learning on all the three datasets. Second, LSGANs generate higher quality images on the second dataset than the first one. This implies that increasing the variability of the dataset can improve the generated image quality and relieve the mode collapse problem. Based on this observation, applying data augmentation such as shifting, cropping, and rotation is an effective way of improving GANs learning.

\subsubsection{Evaluation with Gradient Penalty}
\label{sec:stability_with_gp}
Gradient penalty has been proven to be effective in improving the stability of GAN training ~\cite{Kodali2017,Gulrajani2017}. To compare with WGANs-GP, which is the state-of-the-art GAN model in stability, we adopt the gradient penalty in ~\cite{Kodali2017} for LSGANs and set the hyper-parameters $c$ and $\lambda$ to $30$ and $150$, respectively. For this experiment, our implementation is based on the official implementation of WGANs-GP. We follow the evaluation method in WGANs-GP: to train with six difficult architectures including 1) no normalization and a constant number of filters in the generator; 2) 4-layer 512-dimension ReLU MLP generator; 3) no normalization in either the generator or discriminator; 4) gated multiplicative nonlinearities in both the generator and discriminator; 5) tanh nonlinearities in both the generator and discriminator; and 6) 101-layer ResNet for both the generator and discriminator. The results are presented in Fig. \ref{fig:cmp_wgan}, where the generated images by WGANs-GP are duplicated from ~\cite{Gulrajani2017}. We have the following two major observations. First, like WGANs-GP, LSGANs-GP also succeed in training for each architecture, including 101-layer ResNet. Second, LSGANs-GP with 101-layer ResNet generate higher quality images than the other five architectures.

\subsection{Comparison of Two Parameter Schemes}

As stated in Section \ref{sec:quantitative}, $\text{LSGANs}_{(-110)}$ perform better than $\text{LSGANs}_{(011)}$ for the FID-based experiment. In this experiment, we show another comparison between the two parameter schemes. We train $\text{LSGANs}_{(-110)}$ and $\text{LSGANs}_{(011)}$ on SVHN ~\cite{Netzer2011} dataset using the same network architecture. Fig. \ref{fig:svhn} shows the dynamic results of the two schemes. We can observe that $\text{LSGANs}_{(-110)}$ shows faster convergence speed than $\text{LSGANs}_{(011)}$. We also evaluate the two schemes on the LSUN-bedroom and cat dataset, and similar results are observed.

\subsection{Suggestions in Practice}
Based on the above experiments, we have the following suggestions in practice. First, we suggest using $\text{LSGANs}_{(-110)}$ without gradient penalty if it works, because using gradient penalty will introduce additional computational cost and memory cost. Second, we observe that the quality of generated images by LSGANs may shift between good and bad during the training process, which is also indicated in Fig. \ref{fig:fid}. Thus we suggest to keep a record of generated images at every thousand or hundred iterations and select the model manually by checking the image quality. Third, if LSGANs without gradient penalty fail, we suggest using LSGANs-GP and set the hyper-parameters according to the suggestions in literature ~\cite{Kodali2017}. In our experiments, we find that the hyper-parameter setting, $c=30$ and $\lambda=150$, works for all the tasks.

\section{Conclusions and Future Work}
\label{sec:conclusion}
In this paper, we have proposed the Least Squares Generative Adversarial Networks (LSGANs) to overcome the vanishing gradients problem during the learning process. The experimental results show that LSGANs generate higher quality images than regular GANs. Based on the quantitative experiments, we find that the derived objective function that yields minimizing the Pearson $\chi^2$ divergence performs better than the classical one of using least squares for classification. We also conducted three comparison experiments for evaluating the stability, and the results demonstrate that LSGANs perform more stably than regular GANs. We further compare the stability between LSGANs-GP and WGANs-GP, and LSGANs-GP show comparable stability to WGANs-GP. For the future work, instead of pulling the generated samples toward the decision boundary, designing a method to pull the generated samples toward real data directly is worth further investigation.

\bibliographystyle{IEEEtran}





\end{document}